\documentclass{article}
\usepackage{chngcntr}
\usepackage[utf8]{inputenc} 
\usepackage[T1]{fontenc}    
\usepackage{hyperref}       
\usepackage{url}            
\usepackage{booktabs}       
\usepackage{amsfonts}       
\usepackage{nicefrac}       
\usepackage{microtype}      
\usepackage{xcolor}         
\usepackage{multirow}
\usepackage{graphicx}
\usepackage{hyperref}
\usepackage{subfigure}
\usepackage{caption}
\usepackage{url}
\usepackage{fullpage}
\usepackage{amsmath, amsfonts, hyperref, bbm, amsthm, todonotes}
\usepackage{cleveref, minitoc}
\usepackage[numbers]{natbib}
\newtheorem{theorem}{Theorem}
\newtheorem{lemma}{Lemma}

\newtheorem{assumption}{Assumption}

\newcommand{\Maj}{\mathsf{Maj}}
\renewcommand{\hat}{\widehat}
\usepackage{enumitem}
\usepackage{adjustbox}
\usepackage{wrapfig}
\usepackage{float}
\usepackage{minitoc}
\usepackage{siunitx}

\usepackage{amsmath,amsfonts,bm}

\newcommand{\sign}{\mathsf{sgn}}

\newcommand{\relu}{\sigma}

\newcommand{\defeq}{\vcentcolon=}

\newcommand{\dsame}{\mathcal{D}_{\text{same}}}
\newcommand{\ddiff}{\mathcal{D}_{\text{diff}}}
\newcommand{\dlam}{\mathcal{D}_{\lambda}}

\newcommand{\duni}{\mathcal{D}_{\text{U}}}
\newcommand{\sigmoid}{\phi}

\def\eqref#1{equation~\ref{#1}}

\def\1{\bm{1}}

\DeclareMathAlphabet{\mathsfit}{\encodingdefault}{\sfdefault}{m}{sl}
\SetMathAlphabet{\mathsfit}{bold}{\encodingdefault}{\sfdefault}{bx}{n}

\newcommand{\E}{\mathbb{E}}

\newcommand{\Dist}[1]{\mathcal{D}_{\mathsf{#1}}}

\usepackage{mathtools}

\title{Complexity Matters: Dynamics of Feature Learning in the Presence of Spurious Correlations}

\author{%
  GuanWen Qiu, Da Kuang, Surbhi Goel \\
University of Pennsylvania\\
\texttt{\{guanwenq, kuangda, surbhig\}@seas.upenn.edu}}
\date{}

\begin{document}
\doparttoc 
\faketableofcontents 

\maketitle

\def\doublecolumn{0}

\begin{abstract}
  Existing research often posits spurious features as \textit{easier} to learn than core features in neural network optimization, but the impact of their relative simplicity remains under-explored. Moreover, studies mainly focus on end performance rather than the learning dynamics of feature learning. In this paper, we propose a theoretical framework and an associated synthetic dataset\footnote{The code for introduced datasets is made public at:\url{ https://github.com/NayutaQiu/Boolean_Spurious}} grounded in boolean function analysis. This setup allows for fine-grained control over the relative complexity (compared to core features) and correlation strength (with respect to the label) of spurious features to study the dynamics of feature learning under spurious correlations. Our findings uncover several interesting phenomena: (1) stronger spurious correlations or simpler spurious features slow down the learning rate of the core features, (2) two distinct subnetworks are formed to learn core and spurious features separately, (3) learning phases of spurious and core features are not always separable, (4) spurious features are not forgotten even after core features are fully learned. We demonstrate that our findings justify the success of retraining the last layer to remove spurious correlation and also identifies limitations of popular debiasing algorithms that exploit early learning of spurious features. We support our empirical findings with theoretical analyses for the case of learning XOR features with a one-hidden-layer ReLU network.
\end{abstract}

\section{Introduction}

There is increasing evidence \citep{geirhos_shortcut_2020, zhou_examining_2021, geirhos_imagenet-trained_2022, xiao_noise_2020, mccoy_right_2019} indicating that neural networks inherently tend to learn \textit{spurious} features in classification tasks. These features, while correlated with the data label, are non-causal and lead to enhanced training and in-distribution performance. However, this inherent tendency overlooks core or invariant features that are crucial for robustness against distribution shifts. This phenomenon is attributed to the relative simplicity of spurious features compared to core features, reflecting a \textit{simplicity bias} in neural network training \citep{geirhos_shortcut_2020, shah_pitfalls_2020, rahaman_spectral_2019, nakkiran_sgd_2019, xue_which_2023}, where networks inherently prefer simpler features over more complex, yet essential ones. Interestingly, recent empirical work \citep{kirichenko_last_2023,izmailov_feature_2022} has shown that despite this bias and the compromised predictive performance, standard neural network training does in fact learn the harder core features in its representation, as long as the spurious correlation is not perfect. However, a fine-grained understanding of the impact of ``simplicity'' of the spurious features on the learning of the robust features has remained unexplored. Moreover, a precise definition of simplicity that accounts for computational aspects of learning is lacking.\looseness=-1

In our work, we characterize the impact of the relative complexity of spurious features and their correlation strength with the true label on the dynamics of core feature learning in neural networks trained with (stochastic) gradient descent. To ground our exploration, we introduce a versatile framework and corresponding synthetic datasets based on the rich theory of boolean functions (see \cref{background} for a quick review). We quantify simplicity/complexity using the \textit{computational time/pattern} of learning the different features (represented as boolean functions) by gradient-based training, and subsequently study the dynamics of gradient-based learning on these datasets. We focus on two types of boolean functions: \textit{parity} and \textit{staircase} functions \citep{abbe_merged-staircase_2022}. Our key findings are summarized below:

\begin{itemize}
    \item \vspace{-2ex}\textbf{Easier spurious features lead to slower core feature emergence}. We find that the presence of spurious features notably harm the convergence rate of core feature learning when infinite data is available. This is particularly evident in scenarios where two parity functions of differing degrees are being learned; we give a concrete formula that quantify the initial gradient gap between them. Notably, we found even spurious features that are of similar or slightly lower complexity than the core feature can substantially slow down the convergence rate. This delay in convergence manifests as poor robustness of the model when data availability is limited.
    \item \textbf{The common assumption that the learning phase is separated into spurious feature learning and then core feature learning can lead to complete failure of various debiasing algorithm}. \citep{liu_just_2021, liu_avoiding_2023, utama_towards_2020, nam_learning_2020, yaghoobzadeh_increasing_2021} heavily depend on early learning of shortcut features and make an implicit assumption of a clear-cut separation between the learning phases of core and spurious features. We show that this assumption is generally incorrect and provide an insightful counterexample to demonstrate how such algorithms can fail completely. Staircase functions—a category of functions characterized by their hierarchical structure and learning curves similar to those in real datasets—illustrate that both core and spurious features are learned concurrently. The degree to which core and spurious features are learned is influenced by their relative complexity and correlation strength. This observation challenges the effectiveness of widely adopted machine learning algorithms.
    \item \textbf{Spurious features are retained}. We observe that networks retain spurious features in their representations, even after the core feature has been learned sufficiently well. This retention is particularly notable for spurious features with lower complexity compared to the core features. Not only do these spurious features persist in the network's representation, but their corresponding weights in the last layer also remain stable, especially under high confounder strength. We provide theoretical explanation for such phenomenon and show how well the spurious feature is being memorized is closely related to the correlation strength.
    \item \textbf{The network is separated into two distinct subnetworks in learning different features, and Last Layer Retraining (LLR) decreases reliance on the spurious subnetwork}: \citep{kirichenko_last_2023, izmailov_feature_2022} show LLR with balanced dataset is able to improve robustness of the model. While it is clear from the previous works that the core feature can be linearly decoded from the last layer, the mechanism behind this remains elusive.  We demonstrate across numerous datasets that this improvement primarily stems from a reduction in the weights of the last layer that are connected to the spurious subnetwork. This observation is based on the finding that spurious and core representations are disentangled in the last layer.
    
    
\end{itemize}

We use semi-synthetic and real world datasets to validate the above findings and also provide theoretical justifications for these observations using our boolean spurious feature setting.
\subsection{Related Work} \label{Appendix:related}

\paragraph{Datasets for Studying Spurious Correlations.} Numerous datasets have been employed to study learning under spurious correlation. These include synthetic datasets such as WaterBirds \citep{sagawa_distributionally_2020}, Domino Image dataset \citep{shah_pitfalls_2020}, Color-MNIST \citep{zhang_correct-n-contrast_2022}, and a series of datasets proposed in \citep{hermann_what_2020}. It's important to note that these datasets are constructed in an ad-hoc manner, making it challenging to justify the complexity of the spurious features. Real datasets known to contain spurious correlations, such as CivilComments \citep{duchene_benchmark_2023}, MultiNLI\citep{williams_broad-coverage_2018}, CelebA \citep{liu_deep_2015}, and CXR \citep{kermany_identifying_2018}, are also used to evaluate algorithms designed to mitigate shortcut features. A recent work \citep{joshi_towards_2023} points out several problems of existing datasets that has been used to study spurious correlation and evaluating algorithm performances. Our observation provide further support for their claims (see \cref{Appendix:real_dataset}).

\paragraph{Mitigating Spurious Correlations.}
Learning under spurious correlation can be interpreted as an Out-Of-Distribution (OOD) or group imbalance task, as spurious features divide the dataset into imbalanced groups. Two cases arise: (1) when the spurious attribute is given, popular methods like \citep{sagawa_distributionally_2020, idrissi_simple_2022} can be applied, (2) when the spurious label is unknown during training, various algorithms have been proposed to exploit the phenomenon of simplicity bias \citep{valle-perez_deep_2019, shah_pitfalls_2020, nakkiran_sgd_2019}, which posits that spurious features are learned by the model in the early stages of learning, to upweight underrepresented groups. A representative method of this type is the ``Just Train Twice Algorithm''\citep{liu_just_2021}, where a model is first trained to upweight ``easy'' samples.It is worth noting that almost all algorithms assume a balanced validation dataset for extensive hyperparameter tuning, as observed in \citep{izmailov_feature_2022}.  Another line of work focuses on underspecified tasks where the spurious features are fully correlated with the label and attempt to learn diverse features \citep{teney_evading_2022, lee_diversify_2023}.

\paragraph{Last Layer Retraining.}
A key line of work related to our research is \citep{kirichenko_last_2023, izmailov_feature_2022}, where it is demonstrated that last layer retraining on a biased model with balanced data is enough for achieveing state-of-art result on many benmark datasets. \citep{labonte_towards_2023} further shows this is even true for some benchmark dataset when the spurious data is used. The method essentially runs by first finetune the model then apply logistic regression on a group balanced dataset with heavy regularization term to reweight the last layer. Retraining the last layer has also been explored widely and shown to be highly efficient in other settings, such as long-tail learning \citep{kang_decoupling_2020}, probing inner representations of a model \citep{alain_understanding_2018}, and out-of-distribution learning \citep{rosenfeld_domain-adjusted_2022}. In our study, we assess the quality of the learned representations for both the spurious and core features by evaluating the model's performance after reweighting following \citep{hermann_what_2020}.

\paragraph{Learning Boolean functions with Neural Network(NN).} The problem of learning Boolean functions has long been a fundamental challenge in computational learning theory. A body of work has focused on studying the mechanism of learning the parity function with neural networks in great detail \citep{merrill_tale_2023, daniely_learning_2020, edelman_pareto_2023, barak_hidden_2023}. Another important class of functions, referred to as ``staircase'' functions \citep{abbe_merged-staircase_2022, abbe_sgd_2023} has recently attract great attention and is explored in our study.


%

\section{Boolean Spurious Features Dataset}



To rigorously examine the learning mechanisms of neural networks in the presence of spurious correlations, we propose a dataset that encapsulates features via Boolean functions. We create two boolean features on a set of variables: the core feature which completely predicts the label, and a spurious feature which is correlated to the core feature, but with smaller complexity. Formally, consider two boolean functions
\begin{equation*}
    \underbrace{f_c: \{+1, -1\}^c \to \{+1, -1\}}_{\text{core feature}} \;
    \underbrace{f_s: \{+1, -1\}^s \to \{+1, -1\}}_{\text{spurious feature}}
\end{equation*}
We use $n\defeq c+s+u$ to denote the total dimensions of the vector where the remaining $u$ dimensions are irrelevant variables. For a boolean vector $x$, the coordinates associated with the functions $f_c$ and $f_s$ are denoted by $x_c\in \{+1, -1\}^c,  x_s\in \{+1, -1\}^s$ and we call them core coordinates/features and spurious coordinates/features respectively, while $x_u\in \{+1, -1\}^u$ represents the independent or noise coordinates. The spurious dataset is parameterized by constant $\lambda \in [0,1]$ that represents the confounder strength or correlation of the two features. 

In order to define our spurious dataset, we first form two distributions, and then combine them to form the spurious distribution $\dlam$. With $\duni$ being the uniform distribution on the boolean hypercube, we have 

\begin{itemize}
    \item \vspace{-2ex}$\dsame$, where core and spurious label agree:
    \begin{equation*}
        P_{\dsame}(x) \defeq P_{x \sim \duni}(x \mid f_c(x_c) = f_s(x_s)) 
    \end{equation*}
    
    \item $\ddiff$, where core and spurious label disagree:
    \begin{equation*}
        P_{\ddiff}(x) \defeq P_{x\sim \duni}(x \mid f_c(x_c) \neq f_s(x_s))
    \end{equation*}
    
    \item $\dlam$ where with probability $\lambda$, a sample is drawn from $\dsame$; with probability $1-\lambda$, from $\ddiff$:
    \begin{equation*}
        P_{\dlam}(x) \defeq \lambda P_{\dsame}(x) + (1-\lambda) P_{\ddiff}(x)
    \end{equation*}
\end{itemize} 

\begin{figure*}[t]
    \centering
    \includegraphics[width=0.8\linewidth]{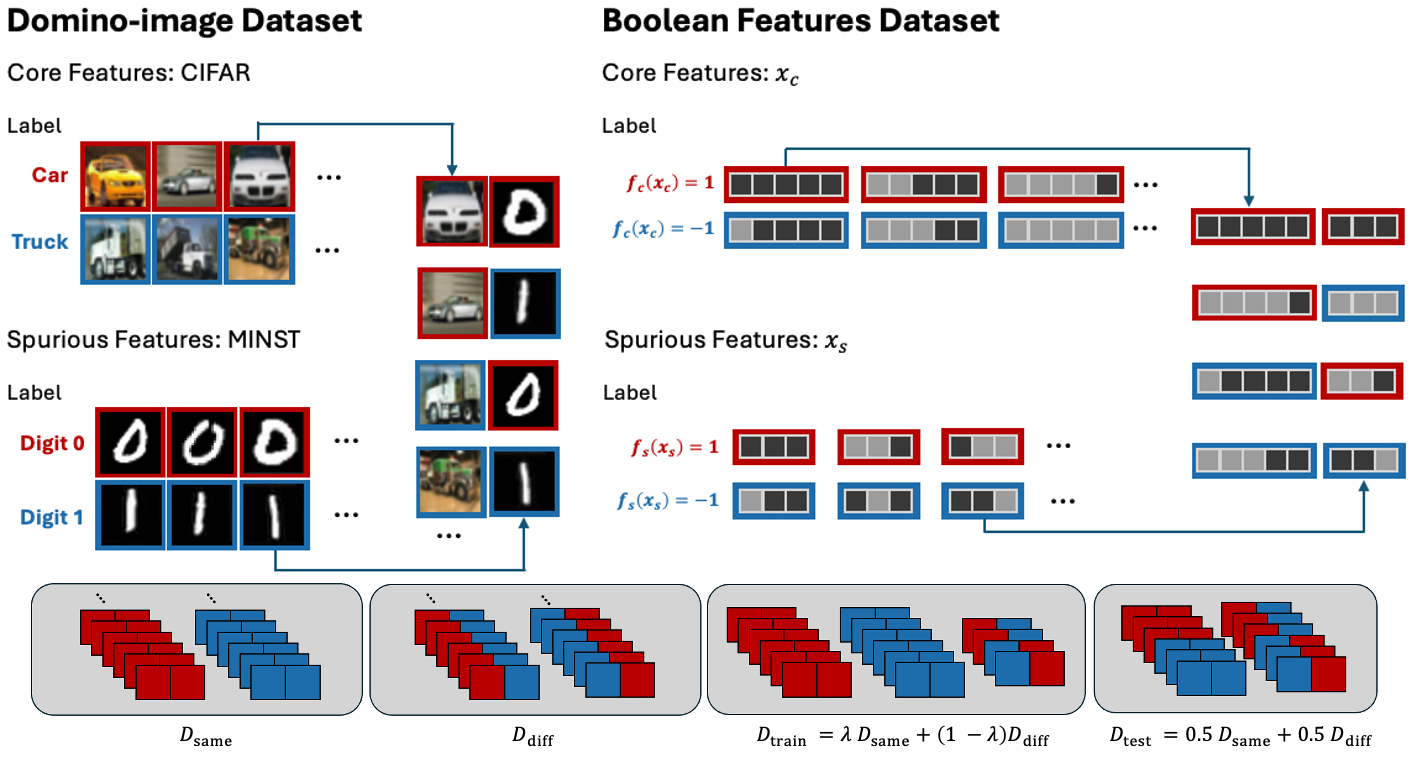}
    \caption{A comparison of our dataset with the domino image dataset. Here $\lambda=0.75$. We take both $f_s$ and $f_c$ to be parity function. Dark grey square on a boolean vector denote $1$ and light grey square denote $-1$.}
    \label{fig:Dsame}
\end{figure*}

 We assume that for both $f \in \{f_s,f_c\}$, the probabilities $P_U(f(x) = 0) > 0$ and $P_U(f(x) = 1) > 0$ (to ensure that the conditional probabilities are well-defined) and $\E_{x\sim D_U}[f(x)] = 0$ (to be unbiased). A sample in our spurious dataset is $(x, y=f_c(x_c))$ where $x \sim \dlam$. Without loss of generality, we assume $\lambda \ge 0.5$, since $\lambda <0.5$ is symmetric with the spurious feature being $-f_s(x_s)$. In simple words, we can view $\dsame$ as upsampling the distribution that has $f_s(x) = f_c(x)$ to add spurious correlations. It is helpful to recognize that $x_s-f(x_s)-y=f(x_c)-x_c$ forms a markov chain. We show some important properties of this distribution in \cref{dist_property}. It is noteworthy that our proposed framework/dataset also satisfies the five constraints proposed in \citep{nagarajan_understanding_2021} to be considered as a ``easy-to-learn'' OOD task\footnote{The second criterion \textbf{Identical Invariant Distribution} made in \citep{nagarajan_understanding_2021} may break when $f_s(x_s)$ is biased.}.

 As is common in these datasets, we define the majority group of samples as $\{x:f_c(x_c)=f_s(x_s)\}$ where the core and spurious features agree ($\dsame$) and the minority group as $\{x:f_c(x_c) \neq f_s(x_s) \}$ where the core and spurious features disagree ($\ddiff$). 

\paragraph{Boolean Functions: Parity and Staircase.}
We focus our study on two choices of the spurious and core features: one where both core and spurious functions are parity functions $f(x) = \chi_d(x) \defeq \prod_i^{n} x_i$, and another where they both take the form of leap 1 degree $d$ staircase functions \citep{abbe_merged-staircase_2022} as $f(x) = x_1 + x_1x_2 + x_1x_2x_3 + \ldots + x_1\ldots x_d$\footnote{These are referred to as staircase functions since the training curves look like a staircase where the features are learned one by one from $x_1$ to $x_d$.}. To adapt the staircase feature to its boolean counterpart, we define the $d$ degree threshold staircase functions,
\begin{equation*}
    sc_d(x) \defeq \begin{cases}
         1 & \text{if } x_1 + x_1x_2 + ... + x_1...x_d \geq 0 \\
         -1 &  \text{else}.
     \end{cases}
\end{equation*}
We show that threshold staircase function have the same structure property as the non-threshold version (see \cref{lemma_staircase}) and similar learning dynamic under cross entropy loss as their unthresholded versions (see \cref{fig:2_dynamics}). Note that $sc_d$ is unbiased when $d$ is an odd number. 

 It is evident that increasing the parameter $d$ in either of these function classes increase their complexity, which in turn is reflected in the convergence rate of the model. Our choice for these particular cases is strategic: both functions offer a solid ground for theoretical analysis, having been extensively examined in the context of deep learning \citep{barak_hidden_2023, daniely_learning_2020, edelman_pareto_2023, abbe_merged-staircase_2022, abbe_sgd_2023}, and despite the same degree $d$, parity features are computationally much harder (exponential in $d$) to learn than staircase functions (polynomial in $d$) \citep{abbe_sgd_2023}. Furthermore, staircase functions are arguably more representative of the intricacies of feature learning on real dataset due to their hierarchical structure, where learning lower degrees aids in advancing to higher ones, and their learning loss curves more closely resemble those encountered in real dataset.
 
\begin{figure*}[t]
    \centering
    \vskip 0.2in
    \includegraphics[width=\linewidth]{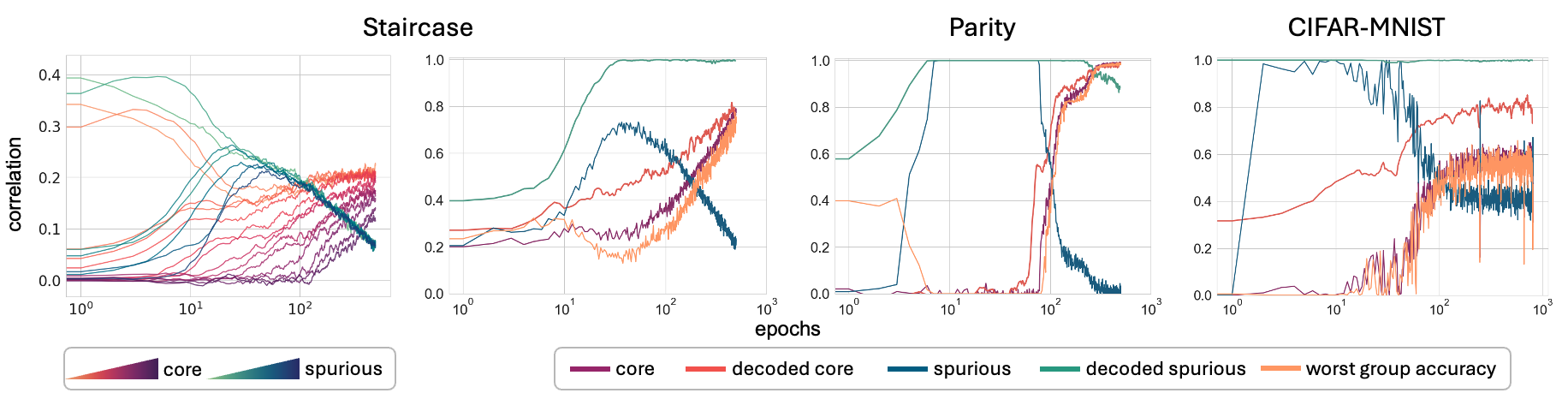}
    \vskip -0.15in
    \caption{
    Core/spurious correlation and decoded correlation dynamics of different datasets. Leftmost figure shows the fourier coefficients of both the spurious and core function are fitted from low (light color) to high (deep color) for the staircase function. All of the experiments have $\lambda = 0.9$. Staircase: $\deg(f_s)=7, \deg(f_c)=14$; Parity: $\deg(f_s)=4, \deg(f_c)=10$; CIFAR-MINST: (c) Truck-car (s) 01.}
    \label{fig:2_dynamics}
\end{figure*}

\paragraph{Why this spurious dataset? }
Numerous works have proposed different theoretical and experimental setups to study spurious correlation. It is noteworthy that spurious features are often used interchangeably with `shortcut' or `easier' features. However, different works have drastically different notions to encapsulate the easiness of a feature. For example, \citep{shah_pitfalls_2020} examines features along different dimensions, quantifying simplicity by the number of linear segments needed for perfectly separating the data. \citep{wen_toward_2021, yang_identifying_2023, sagawa_investigation_2020, chen_understanding_2023} encapsulate both spurious and core features as 1-bit vectors, gauging simplicity by the amount or variance of noise applied to each feature. Despite our framework bearing resemblance to previously proposed notions of simplicity, we distinguish ourselves by: (1) employing non-linear features for both spurious and core attributes, (2) providing a more general notion of feature complexity and allowing us to explore functions with different properties (3) providing a modular, lightweight implementation of our dataset. Additionally, our dataset allows us to provide theoretical explanation for numerous observed behaviours in learning dynamics under spurious correlation. 

We observe that popular semi-synthetic spurious datasets such as Waterbirds\citep{sagawa_distributionally_2020}, Colorful-MNIST \citep{zhang_correct-n-contrast_2022}, and Domino-Image \citep{shah_pitfalls_2020} share characteristic learning dynamics shown on the boolean feature datasets (see \cref{fig:2_dynamics}).  Therefore, our dataset serves as a good proxy to evaluate algorithms developed to deal with spurious features.
Beyond capturing behaviors of existing dataset, our dataset additionally provides precise control over the complexity and structure of the spurious and core functions, which is under-explored in prior datasets.


\section{Empirical Findings}
Here, we provide a comprehensive evaluation of a two-layer\footnote{Neural Networks with more layers share the same behavior. See Figure \ref{fig:depth}.} neural network (width 100) optimized using Batch Stochastic Gradient Descent with the cross entropy loss on the boolean features dataset under the online setting. The exact experimental setup can be found in the Appendix \ref{Appendix:Exp_configs}. 
We emphasize that our main focus here is on the online setting, and we provide more experimental findings regarding limited dataset size in the appendix. We mainly focus on two metrics to measure feature learning: 

\textit{Core and spurious correlation}: correlation between the model and core or spurious feature are measured by $\E_{x\sim \Dist{unif}}[f(x)\mathsf{sgn}(h(x))]$ where $\mathsf{sgn}$ is defined to be the sign function, $h$ is the model, and $f$ is either $f_s$ or $f_c$. Note $\Dist{unif}$ is a group balanced distribution as the functions we studied are unbiased. Thus the core correlation is exactly the mean group accuracy, a metric extensively used in the literature. Additionally, since the spurious correlation is symmetric in our setting, the core correlation closely matches the worst group accuracy.

\textit{Decoded core and spurious correlation} \citep{kirichenko_last_2023,hermann_what_2020, rosenfeld_domain-adjusted_2022, alain_understanding_2018}: we first retrain the last layer of a model with logistic regression to fit either the spurious or core function using a group balanced dataset. Then measure the corresponding correlation as above. The decoded correlation metric is used to capture the extent to which a feature's representation has been effectively learned by the model.

Although our primary focus is on the spurious Boolean dataset, we emphasize that our findings closely align with observations from other semi-synthetic datasets such as Waterbirds \citep{sagawa_distributionally_2020}, CMNIST\citep{zhang_correct-n-contrast_2022}, and Domino Image\citep{shah_pitfalls_2020}. For more detailed exploration of these datasets, see appendix \ref{Appendix:more experiments}.

\begin{figure*}[t]
    \centering
    \vskip 0.2in
    \includegraphics[width=0.9\linewidth]{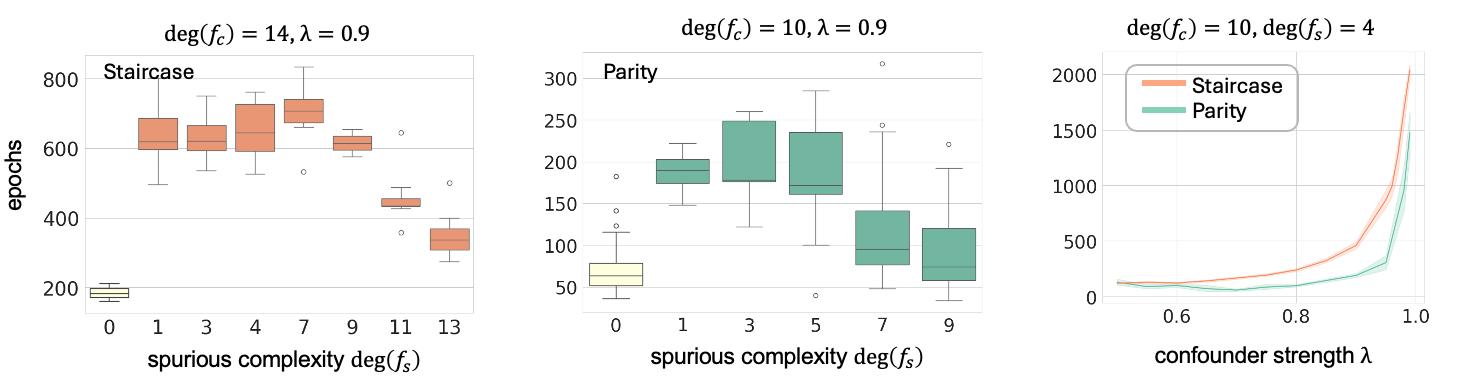}
    \vskip -0.15in
    \caption{Influence of confounder strength and complexity of spurious correlation on learning of core features. The $y$-axis shows the number of epochs required to reach $0.95$ core correlation. The $0$ degree bar indicate the epochs required to learn core feature when spurious correlation is not present i.e $\lambda=0.5$. Each bar of the left two plots is based on 30 repetitions of experiments.}
    \label{fig:impact_corr_comp}
\end{figure*}
 \paragraph{(R1) Simpler spurious features and higher correlation strength slow down the convergence rate of core feature learning} (Figure \ref{fig:impact_corr_comp}). We observe a concave U-shaped phenomenon in the relationship between the complexity of spurious features and convergence time, where lower complexity features slow down convergence. Remarkably, even when the spurious feature approaches the complexity of the core feature, the model's performance is still adversely affected by its presence. Additionally, we find that slower convergence in learning the core feature leads to poorer overall performance on limited-size datasets (see Table \ref{appendix:fig_finite_staircase}). This suggests that the existence of a spurious feature impact the sample complexity required for learning the core feature.

Our investigations indicate that the learning process remains relatively insensitive to the confounder strength until a certain threshold is reached. Beyond this point, there is a sudden and substantial increase in the computational time required to learn the core feature\footnote{Note that our experiments show that even when the confounder strength is as high as 0.99, the model eventually fits the core function perfectly. Refer to Appendix \ref{Appendix:confounder_strength} for more details}.
We hypothesize that this threshold phenomenon can be attributed to two factors. Firstly, different features possess varying learning signal strengths. In the simplest case, exemplified by the parity function, differences in gradient signals for features with different complexities are noticeable from initialization. The gradient of a spurious feature can only surpass that of the core feature if the spurious correlation exceeds a certain threshold value. Secondly, as we will explore later, when $\lambda$ is high, it becomes significantly more challenging for gradient descent to "unlearn" spurious neurons.

\paragraph{(R2)  Spurious and core features are learned by two separate sub-networks} (Figure \ref{fig:neurons}). There exists a classification of neurons into two groups, ``spurious neurons'' which have larger weights on the spurious index and ``core neurons'' which have larger weights on the core index in the late stage of learning. For both parity and staircase tasks, almost all spurious neurons remain focused on spurious coordinates, while core neurons, at the start, do not focus on spurious coordinates and gradually develop an emphasis on core coordinates. See Appendix \ref{Appendix:neurons} for more detail. \looseness=-1

In vision tasks, it becomes more challenging to identify spurious or core neurons. To address this, we retrain the last layer of the neural network to learn either the spurious or core function separately. We observe that the set of neurons with significant weights in both trials is indeed very small, suggesting that neurons are separated into two distinct networks, similar to the spurious Boolean case. Our studies indicate that non-causally related feature representations are perhaps disentangled (at least in the last layer) from the outset, which aligns with a common goal in the fairness and model explainability literature \cite{locatello_fairness_2019, higgins_towards_2018}. Therefore, it is of future interest to understand what conditions are sufficient for a model to learn disentangled representations under common training procedures.
\vskip -0.2in 

\begin{figure*}[t]

    \centering        
    \includegraphics[width=\linewidth]{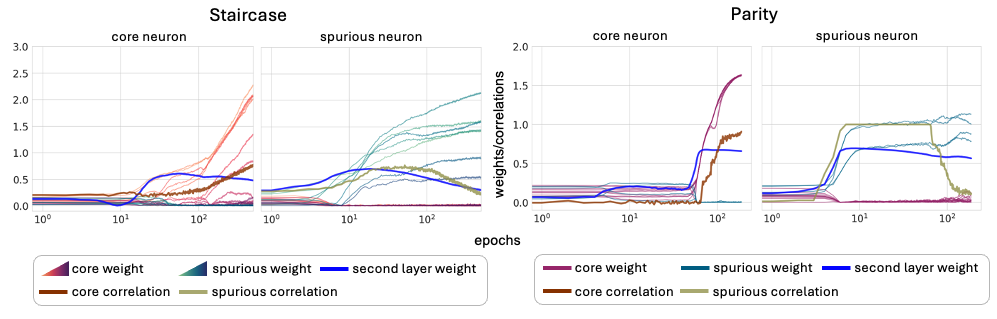}

    \caption{Each plot here shows the weight dynamic throughout training within a single selected neuron and each curve here corresponding to the weight dynamic on a single coordinate. The left two plots are for staircase with $\text{deg}(f_s)=7, \text{deg}(f_c)=14$ and the right two plots are for parity with $\text{deg}(f_s)=4, \text{deg}(f_c)=10$. For both experiment, $\lambda=0.9$. We see the neurons are separated into core and spurious neurons. Spurious neurons remain focus on learning spurious feature and core neuron eventually emerge and learns the core feature. }
    \label{fig:neurons}

\end{figure*}




\paragraph{(R3) Spurious correlation strength determine how well the spurious feature is memorized.} (\cref{fig:2_dynamics}, \cref{fig:mem_spurious_feature}). When $\lambda$ is high, the decoded spurious correlation and the total weight within the spurious subnetwork remains high even after extended training, both in the hidden and last layers. This phenomenon persists even when regularization is applied to the loss function. 

 Notably, when $\lambda$ is low, around 0.6 in our cases, the decoded spurious correlation and total weights on the spurious subnetwork decrease over time as the core feature is learned. In both cases, the learning process for spurious features plateaus when the core correlation starts to exceed the spurious correlation. Our observation thus illustrates another kind of in-distribution forgetting that occurs during training, contrasting with the established catastrophic forgetting, which happens when training on out-of-distribution (OOD) data. Thus to learn diverse features, it could be beneficial to identify and freeze such spurious neurons adaptively as have been done in \cite{kirkpatrick_overcoming_2017, ye_freeze_2023}.

\begin{figure}[t]
    \centering
    \setlength{\belowcaptionskip}{5pt} 
    \includegraphics[width=0.7\textwidth]{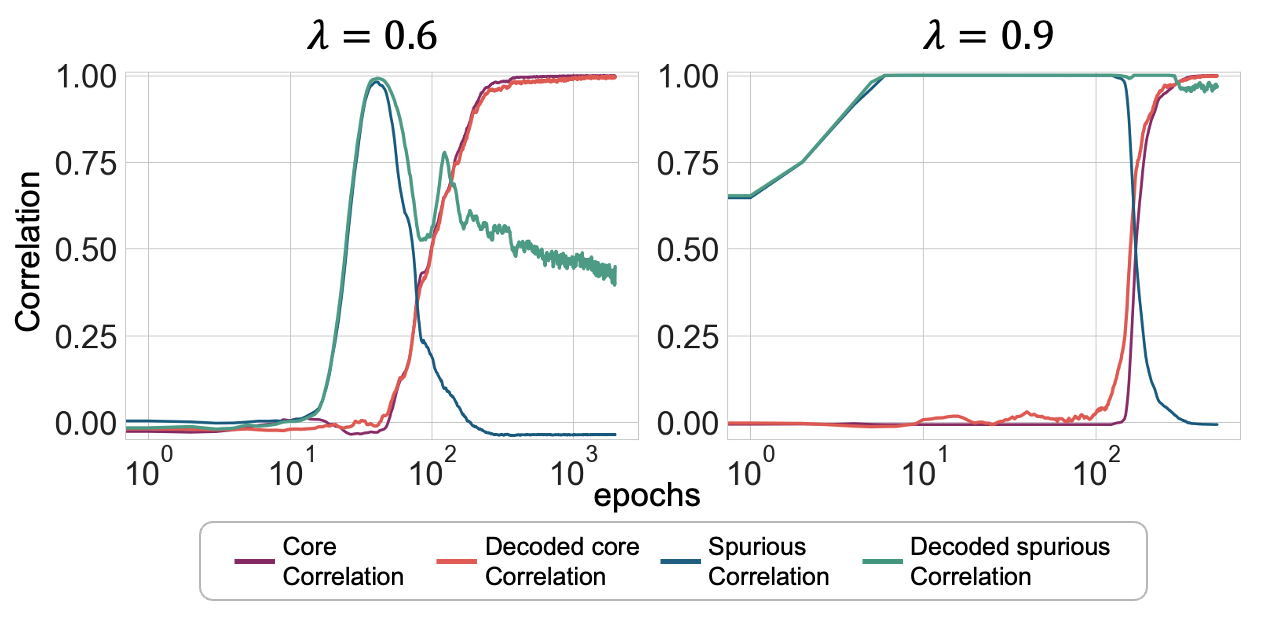}
    \caption{The two plots are produced by running experiments on parity cases under different $\lambda$, the focus here is on the decoded spurious correlation. The experiments are run under $\text{deg}(f_s)=4, \text{deg}(f_c)=10$. Left: we see when $\lambda$ is low, the spurious feature is being forgotten in the later stage of training. Right: when $\lambda$ is relatively high, the spurious feature is memorized once it is learned.}
    \label{fig:mem_spurious_feature}
\end{figure}
\paragraph{(R4) Last Layer Retraining works by decreasing reliance on spurious subnetwork.}
We observed that last layer retraining consistently improves the worst group accuracy or core function correlation, with the most significant performance boost occurring during the early stages of training (Figure \ref{fig:2_dynamics}). This improvement is attributed to a substantial decrease in the ratio of second-layer weights between spurious neurons and core neurons (Table \ref{table:core_spurious_ratios}) which is a consequence of R2 and R3.

Notably, our findings align with those in \citep{labonte_towards_2023}, where we observed that even retraining the last layer on the training dataset (with heavy $l1$ regularization) significantly enhances robustness. Furthermore, we found the performance boost is most significant when just a small amount of group-balanced data is used for LLR (\cref{fig:retrained_core_correlation_hard_staircase}).

\paragraph{(R5) Popular debiasing algorithms fail in more general settings.} (Figure \ref{fig:weakness_algo})
In scenarios where a spurious attribute is absent, debiasing algorithms typically adopt a two-stage approach \citep{liu_just_2021, yang_identifying_2023, nam_learning_2020, kim_biaswap_2021, kim_learning_2022}. They first train an initial model using Early Stop Empirical Risk Minimization (ERM). These algorithms diverge in the second stage, where different heuristics are applied to distinguish and separate minority group data based on insights from the initial model. Implicit in their approach is the assumption of an extreme bias toward simplicity in the spurious feature, expecting the early model to prioritize learning the spurious feature and providing valuable information for segregating minority group samples.

However, our investigation reveals that the benchmark datasets commonly used by these algorithms exhibit a \textbf{crucial dataset bias} towards having a much simpler spurious feature compared to the core feature. This bias creates a distinct separation between the learning phases of spurious and core features, as demonstrated in our parity case, allowing these methods to effectively separate minority groups (see Figure \ref{fig:weakness_algo}). We demonstrate that this separation may not hold true in many cases, particularly with limited datasets and spurious features of similar complexity to the core feature, as illustrated in the staircase case. To emphasize the practicality of our dataset, we introduce a domino-vision dataset with more challenging spurious and core features and shows the model is not able to improve further using early stopped model (see the caption of \cref{fig:real} for dataset configuration). To further assess the effectiveness of these debiasing algorithms, we employ Jaccard scores defined as $\frac{|A \cap B|}{|A \cup B|}$ and containment scores $\frac{|A \cap B|}{|A|}$, where $A$ represents the predicted minority group by the algorithm using an early stop model, and $B$ represents the ground truth minority group. These metrics allow us to evaluate the extent to which minority group data is included in the predictions. For a more detailed discussion of the weaknesses of previous debiasing algorithms and their performance on real datasets, please refer to Appendix \ref{Appendix:debiasing_algorithm_limitations}.

\begin{figure*}[ht]
    \centering
    \vskip 0.2in
    \includegraphics[width=1\linewidth]{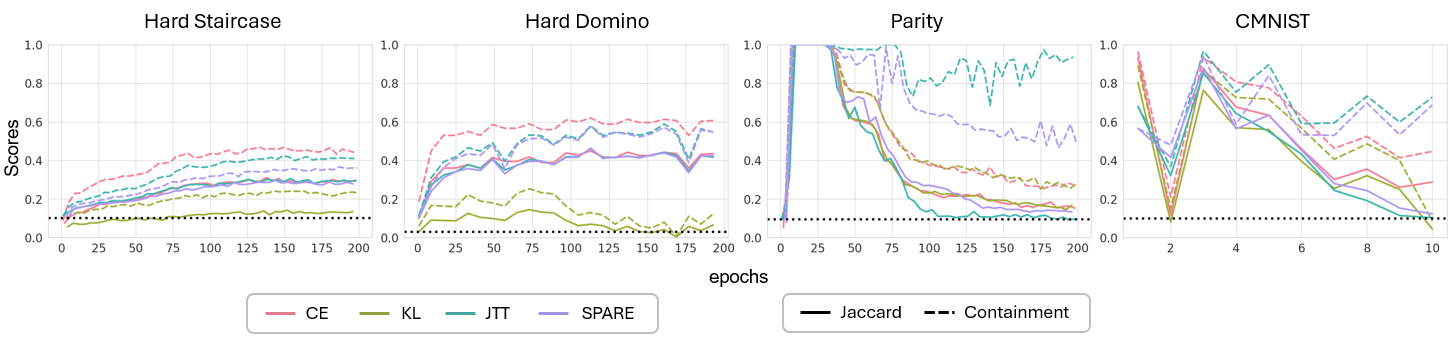}
    \vskip -0.15in
    \caption{
    The two plots on the right demonstrate that debasing methods JTT \citep{liu_just_2021} and Spare \citep{yang_identifying_2023} is able to infer points from the minority group successfully on datasets that share the characteristic of the parity case where there is an early spike in spurious correlation and the spurious feature is much easier than the core feature. While for the more challenging spurious staircase and Domino datasets shown in the left the highest Jaccard score remains below 0.5. We report the worst group accuracy is not improving after upsampling based on the inferred group. See \cref{Appendix:debiasing_algorithm_limitations} for more detail. The black dash line remark the minority group proportion in the training dataset.}
    \label{fig:weakness_algo}
\end{figure*}


\section{Theoretical Explanation}
We will focus on the case of parities for our theoretical analysis. We do not endeavor to present a comprehensive end-to-end analysis of the feature learning dynamics in the general spurious parity case i.e $f_c, f_s$ are different degree of parity functions. We stress that our understanding of end-to-end dynamics of feature learning in the boolean case is still very limited with only recent work \citep{glasgow_sgd_2023} providing an analysis on the end-to-end learning dynamics of 2-parity case. Furthermore, adding spurious correlations, introduces an additional phase of learning which is not tackled by these prior works. However, we hope to provide justification for each observation under certain assumptions at the beginning of each phase. We leave the full end-to-end case analysis to future work.

\paragraph{Setup and Notations.} We consider a two layer neural network with $p$ neurons as $h(x)\defeq \sum_{i=1}^p a_i \relu(w_i^\top x)$ where $\relu$ is the relu activation function $\relu(x) = \max(0,x)$. 
$L_{\mathcal{D}}(h) \defeq \E_{(x,y)\sim \mathcal{D}}[\ell(h(x),y)]$ is defined as the population loss of the model on distribution $\mathcal{D}$. We will use cross entropy loss $\ell(\hat{y},y) \defeq -2\log(\sigmoid(y\hat{y}))$ where $\sigmoid(x) \defeq \frac{1}{1+e^{-x}}$ is the sigmoid function. 

\paragraph{Outline.} We begin by quantifying the "Fourier gap" which represents the difference in population gradient at initialization between the core and spurious features relative to the independent coordinates. The gap immediately implies that, with layer-wise training as proposed in \citep{barak_hidden_2023}, the spurious feature can be learned sufficiently well. After this phase, we estimate the influence of the learned spurious feature on core feature learning by analyzing the change in magnitude of the gradient. Lastly we show that even when the network has learned the core features, learned spurious features must still persist. For detailed calculations and proof, we refer the reader to \cref{Appendix:proofs}.

\paragraph{Spurious Feature is Learned First.}
Following \cite{barak_hidden_2023}, at initialization, the Fourier gap on the spurious and core coordinates relative to the irrelevant coordinates is as follows.
\begin{lemma}[informal]
    Let $\xi_k = \hat{\Maj}([k])$ be the $k$-th Fourier coefficient of the $n=c+s+u$ variable Majority function. At initialization, there is a set of neurons such that the population gradient gap on the variables compared to the irrelevant variables\footnote{These quantities are negative because they refer to the gradient and when we update the weights using gradient descent, the sign will cancel to have a positive contribution on the weights.} are: 
    \begin{enumerate}
        \item \textbf{Spurious Variable}: $-(\lambda - \frac{1}{2})(\xi_{s-1} - \xi_{s+1})$,
        \item \textbf{Core Variable}: $-\frac{1}{2}(\xi_{c-1} - \xi_{c+1})$.
    \end{enumerate}
\end{lemma}
We know that $|\xi_{k}| \approx \Theta \left( n^{-(k-1)/2}\right)$  is monotonically decreasing with $k$, and thus we see the population gradient gap is exponentially higher for the spurious feature than the core feature with respect to the difference in complexity $c-s$ when $\lambda$ is large, which would imply the following:

\begin{theorem}[informal, \citep{barak_hidden_2023}]
    Layer-wise training of the two-layer neural network with SGD is able to learn the spurious parity function up to error $\epsilon$ with $\frac{n^{O(s)}}{\lambda}$ samples and time when $c \gg s$.  
\end{theorem}
After the above layer-wise training, the model would become a Bayes-optimal predictor that depends only on the spurious coordinates. This corresponds to our empirical observation in the parity case when the model is fully correlated with the spurious feature and has not learned the core feature. We have:
\begin{lemma}\label{lem:bayes_opt}
    The Bayes-optimal classifier, with respect to the logistic loss, among the classifiers that depend only on spurious coordinates is $h_s(x_s) = \log\left(\frac{\lambda}{1- \lambda}\right) f_s(x_s)$.
\end{lemma} 

\paragraph{Slow down of Core Feature Learning.} Suppose the network can be divided into a part that has learned this Bayes optimal and the remaining part, we show that having learned spurious feature leads to a reduction in the gradient in the remaining network of the core feature compared to the gradient if there was no spurious feature.
\begin{lemma}[informal]\label{lemma:slow_learning}
    Assume the model can be decomposed into a sub-network $h_s(x_s)$ that is at the Bayes optimal from \cref{lem:bayes_opt} and the remaining network $h(x)$ which is $\approx 0$. Then the gradient with respect to core weights in $h$ is $4\lambda(1-\lambda)$ smaller relative to the gradient if there was no spurious correlation. 
\end{lemma} 
This implies that the core feature will continue to increase but at a slower convergence rate which depends on the correlation strength $\lambda$. Therefore, if the spurious feature is simpler, then the core feature gradient reduces to the lower value earlier, leading to even slower convergence rate. This is consistent with our empirical observations. 


\paragraph{Persistence of Learned Spurious Features.}
Another empirical observation is the persistence of spurious features despite the core feature being learned. Here we present a justification for this in an idealized setup. We consider a neural network that can be decomposed into two sub-networks, \(h_s\) and \(h_c\), such that \(h(x) = h_s(x) + h_c(x)\) where $h_s(x) = \sum_{i \in S} a_i \relu(w_i^\top x_s), h_c(x) = \sum_{i \in C} b_i \relu(v_i^\top x_c)$. The $S, C$ represent the index set of spurious and core neurons respectively. Note that we base this assumption on the empirical finding we have in (R3). 

We further make the assumption that the spurious feature is learned and being Bayes optimal throughout the later stage of training (as observed in our experiments) while the core feature is being learned and the model gives homogeneous response to $x_c$. In particular we assume the following:
\begin{assumption}[Complete correlation to spurious and core features]
    For all $x$, we have $h_s(x)= \gamma_s f_s(x)$ and $h_c(x) = \gamma_cf_c(x)$.
\end{assumption}
The reason we require that the model gives homogesnous response to $x_c$ is due to the property of cross entropy loss which has caused the slow down to be non-linear for different samples if response to $x_c$ is not the same.

We first show that spurious neurons are "dead" as they do not learn core feature in the later stage. \begin{lemma}[informal]
     If $\sum_{i \in C}|w_i| < |w_j|$ for all $j \in S$. Then the gradient on core coordinates of the spurious neuron will be 0.
\end{lemma}
The lemma would also suggest that learned spurious neurons occupy part of the neural network capacity and if the network is of small size, the core feature may not be learned at all. 

On the other hand, the population loss of the model on \(\dlam\) is given by the equation: 
\[
L_{\dlam}(h) = -\lambda \log(\sigmoid(\gamma_c + \gamma_s)) - (1-\lambda) \log(\sigmoid(\gamma_c - \gamma_s)).
\]
The loss function is convex when \(\lambda\) is within the range \([0,1]\) with respect to $\gamma_s$. The optimal point occurs when \(\gamma_c \to \infty\), at which point \(\gamma_s \to 0\). However, in practice, due to bounded iterations of optimization, this ideal scenario is unattainable. Instead, \(\gamma_c\) will inevitably be less than some constant. In such cases, a Bayes optimal model that considers both features will have a positive value for \(\gamma_s\) given that the spurious feature has been learned in the early stage. We conduct numerical experiment (see \cref{fig:gradient_plot}) to illustrates the optimal value of \(\gamma_s\) for varying \(\lambda\) and \(\gamma_c\) and the slow down ratio of core feature gradient. The slow down ratio can be formulated as $\frac{2\lambda(1-\sigmoid(\gamma_s^*+\gamma_c))}{1-\sigmoid(\gamma_c)}$ where the numerator is gradient toward core feature under spurious distribution and the denominator is the same when spurious and core feature is uncorrelated or $\lambda=0.5$. We use $\gamma_s^*$ to denote the optimal value of spurious margin under a fixed $\gamma_c$ and $\lambda$.

\begin{figure}
    \centering
    \includegraphics[width=0.7\textwidth]{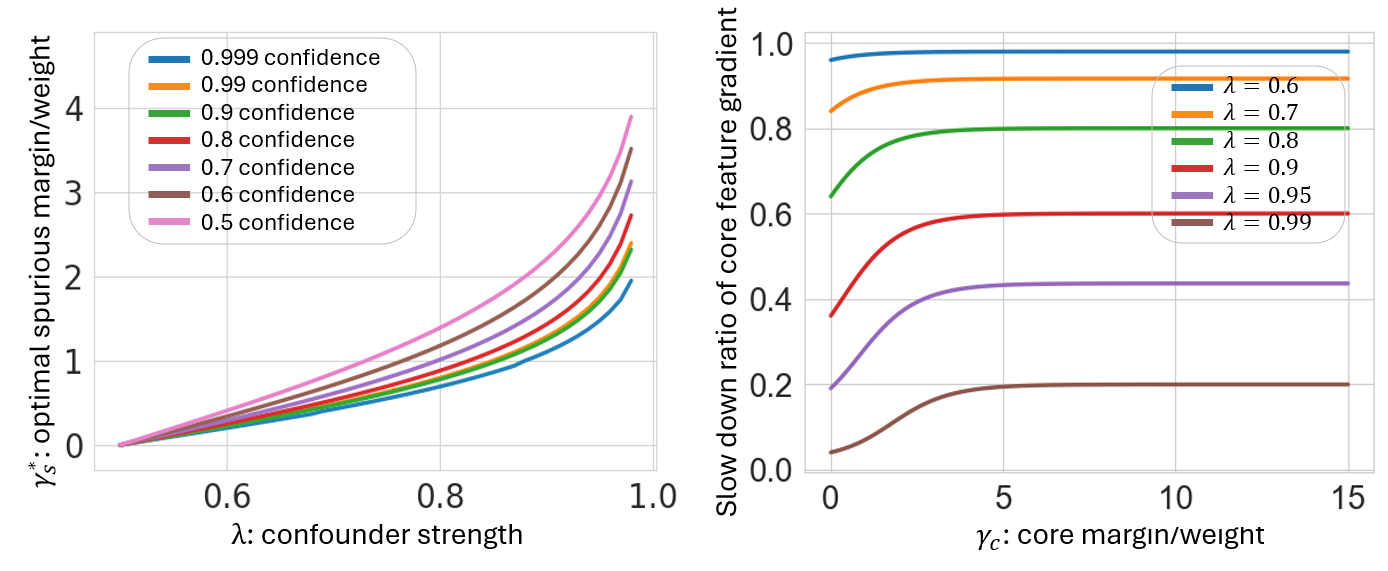}
    \caption{Optimal values of spurious margin under different $\lambda$ with varying confidence core margin. The left plot shows how spurious weights changes with varying $\lambda$ and confidence on core feature. The plot implies retraining a model with more balanced data points would yield most significant improvement when $\lambda$ is high and core feature is learned poorly. The right plot shows how core feature learning is slowed down by the learned spurious feature. Note that under the cross-entropy loss, the margin represents the confidence level of a model, hence $\textbf{confidence} = \phi(\gamma) \leq 1$. For $\gamma_c=15$, we have $\textbf{confidence} > 0.99999$, which is likely to be an empirical upper bound for most iterative optimization methods. The plots are generated by fixing $\lambda$ and $\gamma_c$ at various values and optimizing $\gamma_s$ with respect to $L_{\dlam}$.}
    \vskip -0.2in
    \label{fig:gradient_plot}
\end{figure}
The left plot in \cref{fig:gradient_plot} offers insights into the efficacy of LLR (or data balancing retraining in general) when both the spurious and core features can be adequately learned. Incorporating minority group data points effectively diminishes the convergence point of the spurious margin by reducing $\lambda$, subsequently decreasing the weights of the spurious subnetwork, as evidenced by our findings. Additionally, the plot illustrates that LLR yields more pronounced improvements when the confidence in the core feature is low, as observed in (R4) during the early stages of learning. This observation aligns with the findings in \cite{labonte_towards_2023}, where the addition of a small number of minority examples results in the most significant improvement. The right plot demonstrates that the slow down is less significant in the later stage of training although $\lambda$ plays a key role in determine the empirical convergence point. It is noteworthy that the values of $\gamma_s$ at $\gamma_c = 0$ are equal to the value calculated in \cref{lem:bayes_opt}, thereby offering an estimate of the weight dynamics within the spurious network after \cref{lem:bayes_opt}.



\section{Conclusion}

Our study uses a detailed set of experiments on the boolean feature dataset to better understand how spurious features affect the learning of core features. We show that the dataset is a valuable tool in highlighting the shortcomings of previous algorithms that only perform well with simpler, less challenging spurious features. Our dataset assumes that spurious and core features are completely separate, which might not always be the case in real world datasets, thus improving our dataset to more closely mirror real-world complexities is an exciting opportunity for future research.

On the theoretical side, studying the end-to-end dynamics with spurious features is a technically interesting problem, and may require new tools. Our analysis currently assumes certain well-motivated conditions to bypass some parts of the unknown underlying feature learning process, and we believe that rigorously proving them would be challenging. Notably, showing that the spurious and core networks remain disjoint seems particularly non-trivial to show. A key finding of our analysis is the persistence of spurious feature weights, which tend to converge to significant values when $\lambda$ is high even if the model has learned the core feature confidently. Our investigation reveals that retraining the model with more balanced dataset effectively reduces the weights on the spurious network. While we refer to the observed correlation of another feature as "spurious" in this context, it's essential to recognize that there are instances where correlated features can be benign and offer valuable insights into the core feature. In such cases, we may aim to learn both the spurious and core features.

\section*{Acknowledgements}
 We thank Ben L. Edelman, Vaishnavh Nagarajan, and the anonymous reviewers of the M3L workshop at NeurIPS and ICML for their helpful suggestions to improve the paper.

\bibliography{reference}
\bibliographystyle{alpha}

\renewcommand{\theequation}{\thesection.\arabic{equation}}

\newpage
\appendix

\addcontentsline{toc}{section}{Appendix} 
\part{Appendix} 
\parttoc 
\newpage

\section{Background}
\subsection{Boolean Function Analysis Background}\label{background}
\counterwithin{theorem}{section}
\setcounter{theorem}{0}
See \citep{odonnell_analysis_2021} for a comprehensive review for boolean function analysis. We only include the most important tools here. \begin{theorem}
    Any boolean function $f:\{+1, -1\}^n \to \mathbbm{R}$ can be decomposed into an orthogonal basis \begin{equation*}
        f(x) = \prod_{S\in[n]} \hat{f}(S)\chi_S(x)
    \end{equation*}
    where $\chi_S(x) = \prod_{i\in S} x_i$ and note we have $E[\chi_S(x)\chi_{S'}(x)] = 0$ for all $S'\neq S$. Thus we have $\hat{f}(S) = E[f(x)\chi_S(x)]$.
\end{theorem}

\subsection{Properties of Threshold Staircase Function}
\begin{lemma}\label{lemma_staircase}
     For every term $S$ inside $sc_d(x)$, with $d\geq 2$, we have \begin{equation*}
         \Hat{sc_d}(S) = \begin{cases}
             \frac{\binom{d-1}{\frac{d}{2}-1}}{2^{d-1}} & \text{if } d \text{ even} \\ 
             \frac{\binom{d-2}{\frac{d-1}{2}}}{2^{d-2}} & \text{if } d \text{ odd}
         \end{cases}
     \end{equation*}
 \end{lemma}

\begin{proof}
    We outline the proof of staircase function here. Note for each of the term in $x_1+x_1x_2+x_1x_2x_3$, they are independently distributed. Thus for terms in the staircase the fourier coefficent of them for the threshold staircase is the same as the majority function for $\Hat{Maj_d}(k)$ where $k=1$. Further, we have $\sum_{s\in S} \Hat{sc_d}(s)^2 \to \frac{2}{\pi}$ as $d \to \infty$ \citep{odonnell_analysis_2021} (ch5.3).
\end{proof}
\subsection{Properties of Boolean spurious distribution}
\counterwithin{lemma}{section}
\setcounter{lemma}{0}
\begin{lemma} \label{dist_property}
    Given a spurious boolean distribution defined previously, if either of the feature is unbiased. Then the distribution is equivalent to a mixture of $2(1-\lambda)\Dist{unif} + (2\lambda-1)\Dist{same}$. 
\end{lemma}
\begin{proof}
    This is equivalent to saying that given any boolean vector $x$, $P_D(x) = 2(1-\lambda)P_{\mathsf{unif}}(x) + (1-2\lambda)P_{\mathsf{same}}(x)$. Note that $P_D(x) = \lambda P_{\mathsf{same}}(x) + (1-\lambda)P_{\mathsf{diff}}(x) = (2\lambda-1) P_{\mathsf{same}}(x)  + (1-\lambda)(P_{\mathsf{diff}}(x) + P_{\mathsf{same}}(x))$. Thus we only need to argue that $P_{\mathsf{diff}}(x) + P_{\mathsf{same}}(x) = 2P_{\mathsf{unif}}(x)$ forms a uniform distribution. And we have \begin{align*}
        &P_{\mathsf{diff}}(x) + P_{\mathsf{same}}(x) \\
        =&\frac{P_{x\sim U}(x=x, f_s(x) = f_c(x))}{P(f_s(x)=f_c(x))} + \frac{P_{x\sim U}(x=x, f_s(x) \neq f_c(x))}{P(f_s(x)\neq f_c(x))} \\
        =&I[f_s(x) = f_c(x)] \frac{P_{x\sim U}(x=x)}{P(f_s(x)=f_c(x))} + I[f_s(x) \neq f_c(x)]\frac{P_{x\sim U}(x=x)}{P(f_s(x)\neq f_c(x))} \\
        =&I[f_s(x) = f_c(x)] \frac{P_{x\sim U}(x=x)}{P(f_s(x)=1)P(f_c(x)=1)+P(f_s(x)=-1)P(f_c(x)=-1)} \\
       +&I[f_s(x) \neq f_c(x)] \frac{P_{x\sim U}(x=x)}{P(f_s(x)=1)P(f_c(x)=-1)+P(f_s(x)=-1)P(f_c(x)=1)} \\
       =&I[f_s(x) = f_c(x)] 2 P_{x\sim U}(x=x) + I[f_s(x) \neq f_c(x)] 2 P_{x\sim U}(x=x) \\
       =&2 P_{x\sim U}(x=x)
    \end{align*}
\label{lemma: distribution_transformation}

\end{proof}

\begin{lemma}
    When $f_c(x_c)$ is unbiased, $x_s$ is marginally uniformly distributed on $\dsame, \ddiff, \dlam$. And when $f_s(x_s)$ is unbiased $x_c$ is marginally uniformly distributed on $\dsame, \ddiff, \dlam$
\end{lemma}

\begin{proof}
    Consider $P_{x \sim U}(X_s=x_s |f(X_s)=f(X_c)) = \frac{P(X=x_s)P(f_c(X_c)=f_s(x_s))}{P(f_c(x_c)=-1) P(f_s(x_s)=-1) + P(f_c(x_c)=1) P(f_s(x_s)=1)}$ Suppose $f_s(x_s) = 1$, then we have $\frac{P(X=x_s)P(f_c(x_c)=1)} {P(f_c(x_c)=-1) P(f_s(x_s)=-1) + P(f_c(x_c)=1) P(f_s(x_s)=1)} = P_U(X_s=x_s)$. The equality also hold if we take $f_s(x_s)=-1$ or $\ddiff$. Which would imply this is also true for $\dlam$ Thus we complete the proof.
\end{proof}

\begin{lemma}For spurious function $f_s$ and core function $f_c$ satisfying \cref{dist_property}, we have
\begin{equation*} 
        \E_{(x, y)\sim \mathcal{D}_{\lambda}}[\chi_S(x)y] = \hat{f_c}(S) + (2\lambda-1)\hat{f}_s(S).
    \end{equation*} 
\end{lemma}
\begin{proof}
    We have,
    \begin{align*}
        \E_{\mathcal{D}_{\lambda}}[\chi_S(x)y] &= 2 ( 1- \lambda) \E_{\Dist{unif}}[\chi_S(x)f_c(x_c)] + (2\lambda-1) \E_{\Dist{unif}}[\chi_S(x)f_c(x_c) | f_c(x_c) = f_s(x_s)]\\
        &= 2 ( 1- \lambda)\hat{f_c}(S) + (2\lambda-1)\frac{\E_{\Dist{unif}}[\chi_S(x)f_c(x_c) \mathbbm{1}[f_c(x_c) f_s(x_s) = 1]]}{\E_{\Dist{unif}}[\mathbbm{1}[f_c(x_c) f_s(x_s) = 1]}\\
        &= 2 ( 1- \lambda)\hat{f_c}(S) + (2\lambda-1)\frac{\E_{\Dist{unif}}\left[\chi_S(x)f_c(x_c) \left(\frac{f_c(x_c) f_s(x_s) + 1}{2}\right)\right]}{\E_{\Dist{unif}}\left[\frac{f_c(x_c) f_s(x_s) + 1}{2}\right]}\\
        &= 2 ( 1- \lambda)\hat{f_c}(S) + (2\lambda-1)\frac{\E_{\Dist{unif}}\left[\chi_S(x)(f_s(x_s) + f_c(x_c))\right]}{\E_{\Dist{unif}}[f_c(x_c)] \E_{\Dist{unif}}[f_s(x_s)] + 1}\\
        &= 2 ( 1- \lambda)\hat{f_c}(S) + (2\lambda-1)\left(\hat{f}_c(S) + \hat{f}_s(S)\right)\\
        &= \hat{f_c}(S) + (2\lambda-1)\hat{f}_s(S).
    \end{align*}
\end{proof}
From the above, we have
\begin{equation*}
    \E_{\mathcal{D}_{\lambda}}[\chi_S(x)y] = \begin{cases}
        \hat{f_c}(S) & \text{ if } S \subseteq [c]\\
        (2 \lambda - 1)\hat{f_s}(S) & \text{ if } S \subseteq [s]\\ 
        0 & \text{ otherwise.}
    \end{cases}
\end{equation*}
\section{Theory}
\subsection{Calculation of Gradient Gaps and Revised Proof for Layer Wise Training for Parity Function} \label{Appendix:proofs}

\subsubsection{Setting}
Recall the definition of the boolean task we defined in the draft. We define 1. $x$ the concatenation of three vectors $x_s$, $x_c$, $x_u$. The length of $x_s$ is $s$ and the length of $x_c$ is $c$ with $s << c$. $x_u$ here denote a length $u$ random boolean vector which does not have any correlation with the label. Thus the length of $x$ is $n=s+c+u$ Without loss of generality, we additionally require $c,s$ to be even length and $u$ to be odd length. 2. $f_s$ and $f_c$ are parity functions defined as $\chi^S(x_s) = \prod_{i \in [s]} x_i$ and $\chi^C(x_c) = \prod_{i \in [c]}$. So $\deg(f_s) = s, \deg(f_c)=c$. 3. We then form a spurious distribution as defined in the main paper with confounder strength $\lambda$. 
\subsubsection{Model}
We consider a 1 hidden layer ReLU neural network with $r$ neurons  \begin{equation*}
    f(x) = \sum_{i=1}^r a_i \sigma(w_i^\top x + b_i)
\end{equation*}
where $a_i \in \mathbbm{R}, w_i \in \mathbbm{R}^n$ and $b_i \in \mathbbm{R}$.
We use logistic loss \begin{equation*}
    \ell(y, \hat{y}) = -2 \log(\sigmoid(y\hat{y}))
\end{equation*}
where $\hat{y}$ is the output of the model and $y$ is the true label.

\subsubsection{Initialization} \label{setting:initalization_scheme}
Let us consider the following initialization scheme as in \citep{barak_hidden_2023}
\begin{enumerate}
    \item For all $1 \leq i \leq r/2$, randomly initialize \begin{equation*}
    w_i^{(0)} \sim \mathsf{Unif} ({+1,-1}^n),
    a_i^{(0)} \sim \mathsf{Unif} ({+1,-1}),
    b_i^{(0)} \sim \mathsf{Unif} ({-1 + 1/c,-1 + 2/c, ... , 1-1/c})
\end{equation*}

    \item For all $r/2 < i \leq r$, initialize \begin{equation*}
    w_i^{(0)} = w_{i-r/2}^{(0)},
    a_i^{(0)} = -a_{i-r/2}^{(0)},
    b_i^{(0)} = b_{i-r/2}^{(0)}
    \end{equation*}
\end{enumerate}
Key properties of this initialization scheme are: (1) It is unbiased, since the model output is 0 on all inputs at initialization, and (2) Biases $b$ are set such that they enable computing parity linearly once we have the correct coordinates identified. For the informal lemma in the main paper, we assume $w_i$ is the all 1s vector.

We will first analyze the gradients at initialization, then after spurious feature is learned.

\subsubsection{Population Gradient Gap at Initialization}
Notice that our initalization makes the model output 0 on all $x$. Thus $l(y,\hat{y}) = 0$ and then we have $\nabla_{\hat{y}} l'(y, \hat{y}) = -y$ .

We can now formulate the population gradient at initialization. Without loss of generality, we will assume $\lambda > 0.5$, then $\mathcal{D}_{\lambda}$ is a mixture such that w.p $2(1-\lambda)$ we draw a sample $x$ from the uniform distribution $\mathsf{Unif}(\{+1,-1\}^n)$. And with $2\lambda-1$, we draw a sample from $\Dist{same}$ where we first draw $x_c \sim \mathsf{Unif}(\{+1,-1\}^c)$ and then draw a $x_s \sim \mathsf{Unif}\{x_s |\chi_S(x_s)= \chi_C(x_c)\}$. 

Then the population gradient for weight $w_{i,j}$ is \begin{align*}
    \mathbb{E}_{\mathcal{D}_{\lambda}}&[\nabla_{w_i,j} l(f(x;\theta_0),y)] \\ &=\mathbb{E}_{\mathcal{D}_{\lambda}}[-y\nabla_{w_{i,j}} f(x;\theta_0)]  \\
    &= \mathbb{E}_{\mathcal{D}_{\lambda}}[-ya_i\mathbbm{1}\{w_i^\top x +b_i > 0\}x_j] \\
    &= 2(1-\lambda)\mathbb{E}_{\Dist{unif}}[-ya_i\mathbbm{1}\{w_i^\top x +b_i > 0\}x_j]  + (2\lambda-1)\mathbb{E}_{D_s}[-ya_i\mathbbm{1}\{w_i^\top x +b_i > 0\}x_j] \label{eq:pop grad}\tag{1}
\end{align*}

We will study the two terms separately.

\noindent\textbf{Population Gradient on uniform distribution.}
Set $g_{i,j} = \mathbb{E}_{\Dist{unif}}[-ya_i\mathbbm{1}\{w_i^\top x +b_i > 0\}x_j]$. As long as $w_i \in \{-1,1\}^n$, from \citep{barak_hidden_2023}, we have
\begin{enumerate}
    \item For $j\in[c]$: \begin{equation*}
        g_{i,j} = -\frac{1}{2}a_i \xi_{c-1} \cdot \chi_{[c]\setminus\{j\}}(w_i)
    \end{equation*}
    \item For $j\in [s]\cup[u]$ :\begin{equation*}
        g_{i,j} = -\frac{1}{2}a_i \xi_{c+1} \cdot \chi_{[c] \cup \{j\}}(w_i)
    \end{equation*}
\end{enumerate} where $\xi_k = \hat{\Maj}(S)$ with $|S|=k$.
Thus we have for the first term $g^u_{i,j} = 2(1-\lambda)\mathbb{E}_{\Dist{unif}}[-ya_i\mathbbm{1}\{w_i^\top x +b_i > 0\}x_j] $ \begin{enumerate}
    \item For $j\in[c]$: \begin{equation*}
        g^u_{i,j} = -(1-\lambda)a_i \xi_{c-1} \cdot \chi_{[c]\setminus\{j\}}(w_i)
    \end{equation*}
    \item For $j\in [s] \cup [u]$:\begin{equation*}
        g^u_{i,j} = -(1-\lambda)a_i \xi_{c+1} \cdot \chi_{[c] \cup \{j\}}(w_i)
    \end{equation*}
\end{enumerate}

\noindent\textbf{Population Gradient on label aligned distribution $\Dist{same}$.}
Note for the second term of \ref{eq:pop grad},  \begin{align*}
    &(2\lambda-1)\mathbb{E}_{x_c, x_s \sim D_s}[-ya_i\mathbbm{1}\{w_i^\top x +b_i > 0\}x_j] \\
    &= (2\lambda-1)([P_{x_c \sim u}[\chi(x_c)=-1]\mathbb{E}_{x_c,x_s\sim u}[-ya_i\mathbbm{1}\{w_i^\top x +b_i > 0\}x_j|\chi(x_c)=-1, \chi(x_s)=-1] \\
    &\quad +P_{x_c\sim u}[\chi(x_c)=1]\mathbb{E}_{x_c, x_s\sim u}[-ya_i\mathbbm{1}\{w_i^\top x +b_i > 0\}x_j|\chi(x_c)=1, \chi(x_s)=1]) \\
    &= (2\lambda-1)(\frac{1}{2} \cdot 4 \cdot \mathbb{E}_{x_c,x_s\sim u}[-ya_i\mathbbm{1}\{w_i^\top x +b_i > 0\}x_j \mathbbm{1}\{\chi(x_c)=1, \chi(x_s)=1\}] \\
    &\quad +\frac{1}{2}\cdot 4 \cdot \mathbb{E}_{x_c, x_s\sim u}[-ya_i\mathbbm{1}\{w_i^\top x +b_i > 0\}x_j \mathbbm{1}\{\chi(x_c)=1, \chi(x_s)=1\}) \\
    &= -2a_i(2\lambda-1)\cdot \mathbb{E}_{x_c,x_s\sim u}[yx_j\mathbbm{1}\{w_i^\top x +b_i > 0\} \mathbbm{1}\{\chi(x_c) = \chi(x_s)\}] \label{eq:aligned grad} \tag{2}
\end{align*}
By the initialization scheme, we have $w_i^\top x$ as an integer and $|b|<1$. Thus \begin{equation*}
    \mathbbm{1}\{w_i^\top x +b_i > 0\} = \mathbbm{1}\{w_i^\top x > 0\} 
\end{equation*}
We will first ignore the $yx_j$ component inside the expectation and study the boolean function  $q(x) = \mathbbm{1}\{w_i^\top x > 0\} \mathbbm{1}\{\chi(x_c) = \chi(x_s)\}$. Observe that \[
        \mathbbm{1}\{\chi(x_c) = \chi(x_s)\} = \mathbbm{1}\{\chi(x_c) \chi(x_s)  = 1\} = \mathbbm{1}\{\chi_{[c]\cup [s]} (x) = 1\}. \]
    This gives us \begin{align*}
        q(x) &= \mathbbm{1}\{w_i^\top x > 0\}\mathbbm{1}\{\chi_{[c]\cup [s]} (x) = 1\} \\
        &= \frac{1+\Maj_n(w_i \odot x)}{2} \cdot\frac{1+ \chi_{[c]\cup [s]} (x)}{2}\\
        &= \frac{1}{4} \cdot \left(1 + \chi_{[c]\cup [s]} (x) + \Maj_n(w_i \odot x) + \chi_{[c]\cup [s]} (x)\Maj_n(w_i \odot x)\right) \label{eq:four terms}\tag{3}
    \end{align*}
    We can study the fourier spectrum of each of the term in \ref{eq:four terms} to construct the fourier spectrum of $q$. \begin{enumerate}
        \item For $q_1(x) = \chi_{[c]\cup [s]} (x)$, notice that $\hat{q}(S) = 0$ for all $S\subset[n]$ except when $S=[c]\cup [s]$ where $\hat{q}(S) = 1$.
        \item For $q_2(x) = \Maj(w \odot x)$, notice that  $\Maj(x)$ can be written in its fourier expansion as $\Maj(x)= \sum_{S\subseteq[n]} \hat{\Maj_n}(S) \chi_S(x)$. This gives us \begin{equation*}
            q_2(x) = \Maj(w_i \odot x) = \sum_{S\subseteq[n]} \hat{\Maj_n}(S)\chi_S(w_i\odot x) = \sum_{S\subseteq[n]} \hat{\Maj_n}(S)\chi_S(x)\chi_S(w_i).
        \end{equation*} Thus we have $\hat{q_2}(S) = \chi_S(w_i)\hat{\Maj_n}(S)$.
        \item For $q_3(x) = \chi_{[c]\cup [s]}(x)\Maj_n(w_i\odot x)$, we have \begin{align*}
            \hat{q_3}(S) &= \mathbb{E}[\Maj_n(w_i\odot x)\chi_S(x)\chi_{[c]\cup [s]} (x)] \\
            &= \E_x[\Maj_n(w_i\odot x)\chi_{([s]\cup[c])\Delta S}(x)] \\
            &- \hat{q_2}(([c] \cup [s]) \Delta S) \\
            &= \chi_{([c] \cup [s]) \Delta S}(w_i)\hat{\Maj_n}(([c] \cup [s]) \Delta S)
        \end{align*}
    \end{enumerate} 
    By the orthogonality and linearity of the fourier basis, we thus have \begin{equation*}
        \hat{q}(S) = \frac{1}{4} (\chi_S(w_i)\hat{\Maj_n}(S) + \chi_{([c] \cup [s]) \Delta S}(w_i)\hat{\Maj_n}(([c] \cup [s]) \Delta S)) 
    \end{equation*} for $|S| > 0$ and $S \neq [s]\cup[c]$.
    Now let us put it back in \ref{eq:aligned grad}.\begin{enumerate}
        \item For random index $j\in [u]$, \begin{align*}
             g^s_{i,j} &=-2a_i(2\lambda-1)\cdot \mathbb{E}_{x_c,x_s\sim u}[yx_j\mathbbm{1}\{w_i^\top x +b_i > 0\} \mathbbm{1}\{\chi(x_c) = \chi(x_s)\}] \\
            &= -2a_i(2\lambda-1)\cdot\mathbb{E}_{x \sim u}[\chi_{[c]\cup j}(x) q(x)] \\
            &= -2a_i(2\lambda-1)\cdot\hat{q}([c]\cup j) \\
            &= -\frac{1}{2} a_i(2\lambda-1)\cdot  (\chi_{[c]\cup \{j\}}(w_i)\xi_{c+1} + \chi_{[s]\cup \{j\}}(w_i)\xi_{s+1})
        \end{align*}
        \item For core index $j \in [c]$, in a similar manner, we have \begin{align*}
             &g^s_{i,j}\\
            &=-2a_i(2\lambda-1)\cdot \mathbb{E}_{x_c,x_s\sim u}[yx_j\mathbbm{1}\{w_i^\top x +b_i > 0\} \mathbbm{1}\{\chi(x_c) = \chi(x_s)\}] \\
            &= -2a_i(2\lambda-1)\cdot \mathbb{E}_{x \sim u}[\chi_{[c]\setminus j}(x) q(x)] \\
            &= -2a_i(2\lambda-1)\cdot \hat{q}([c]\setminus j) \\
            &= -\frac{1}{2} a_i(2\lambda-1)\cdot (\chi_{[c]\setminus \{j\}}(w_i)\xi_{c-1} + \chi_{[s]\cup \{j\}}(w_i)\xi_{s+1})
        \end{align*}
        \item For spurious index $j \in [s]$, in a similar manner, we have \begin{align*}
             g^s_{i,j} &=-2a_i(2\lambda-1)\cdot \mathbb{E}_{x_c,x_s\sim u}[yx_j\mathbbm{1}\{w_i^\top x +b_i > 0\} \mathbbm{1}\{\chi(x_c) = \chi(x_s)\}] \\
            &= -2a_i(2\lambda-1)\cdot \mathbb{E}_{x \sim u}[\chi_{[s]\setminus j}(x) q(x)] \\
            &= -2a_i(2\lambda-1)\cdot \hat{q}([s]\setminus j) \\
            &= -\frac{1}{2} a_i(2\lambda-1)\cdot (\chi_{[s]\setminus \{j\}}(w_i)\xi_{s-1} + \chi_{[c]\cup \{j\}}(w_i)\xi_{c + 1})
        \end{align*}
    \end{enumerate}
    
\noindent\textbf{Putting it together.} \label{procedure:find_gradient}
\counterwithout{lemma}{section}
\setcounter{lemma}{0}
We summarize the final population gradient on each type of index.\begin{lemma}[formal] Under the proposed setup, we have the population gradient at initialization for different type of coordinates as below: \begin{enumerate}
    \item For random index $j\in [u]$, \begin{align*}
            g_{i,j} &= g^u_{i,j} + g^s_{i,j} \\
            &= -(1-\lambda)a_i \xi_{c+1} \cdot \chi_{[c] \cup \{j\}}(w_i) -\frac{1}{2} a_i(2\lambda-1)\cdot  (\chi_{[c]\cup \{j\}}(w_i)\xi_{c+1} + \chi_{[s]\cup \{j\}}(w_i)\xi_{s+1}) \\
            &= -a_i\left(\frac{1}{2} \xi_{c+1} \cdot \chi_{[c] \cup \{j\}}(w_i) + \left(\lambda-\frac{1}{2}\right) \xi_{s+1} \cdot \chi_{[s]\cup \{j\}}(w_i)\right) 
        \end{align*}
        \item For core index $j \in [c]$, in a similar manner, we have \begin{align*}
            g_{i,j} &= g^u_{i,j} + g^s_{i,j} \\
            &= -(1-\lambda)a_i \xi_{c-1} \cdot \chi_{[c] \setminus \{j\}}(w_i) -\frac{1}{2} a_i(2\lambda-1)\cdot (\chi_{[c]\setminus \{j\}}(w_i)\xi_{c-1} + \chi_{[s]\cup \{j\}}(w_i)\xi_{s+1}) \\
            &= -a_i\left(\frac{1}{2} \xi_{c-1} \cdot \chi_{[c] \setminus \{j\}}(w_i) + \left(\lambda-\frac{1}{2}\right) \xi_{s+1} \cdot \chi_{[s]\cup \{j\}}(w_i)\right)
        \end{align*}
        \item For spurious index $j \in [s]$, in a similar manner, we have \begin{align*}
            g_{i,j} &= g^u_{i,j} + g^s_{i,j} \\
            &= -(1-\lambda)a_i \xi_{c+1} \cdot \chi_{[c] \cup \{j\}}(w_i) -\frac{1}{2} a_i(2\lambda-1)\cdot (\chi_{[s]\setminus \{j\}}(w_i)\xi_{s-1} + \chi_{[c]\cup \{j\}}(w_i)\xi_{c + 1}) \\
            &= -a_i\left(\frac{1}{2} \xi_{c+1} \cdot \chi_{[c] \cup \{j\}}(w_i) + \left(\lambda-\frac{1}{2}\right) \xi_{s-1} \cdot \chi_{[s]\setminus \{j\}}(w_i)\right)
        \end{align*}
\end{enumerate}
\end{lemma}

\subsubsection{Finding hidden functions that makes gradient descent on the last layer weights recover parity function} \label{appendix:twi_step_proof}
This part is correspond to the lemma 4 claim 5 in the hidden progress paper. What we are going to argue is that there is a set of functions represented by the hidden layer $\relu(\sum_i x - b_i)$ such that there is a set of weight $a_i$ makes the composition function equivalent to a parity function. Formally 
\counterwithin{lemma}{section}
\setcounter{lemma}{0}
\begin{lemma}\label{lemma:1}
    Fix $k$, there is a set of $k+1$ ReLU functions in the form $f_j(x) = \relu(\sum_i^n x_i + b_j)$ where $b_j = \{k+1, k-1, k-3, ..., -k + 3, -k +1 \}$ and a set of weights $u_j$ with $\|u\|_2 \leq \sqrt{k}$ modeling the k-degree parity function by $f(x) = \chi_k(x) = \sum_j^{k+1} f_j(x)$.  
\end{lemma}
\begin{proof}
    This is best illustrated by an example. Consider $k=5$, if there is one $-1$ in $x$, then $f_1(x) = 9, f_2(x) = 7, f_3(x)=5, f_4(x)=3, f_5(x)=1, f_6(x)=0$. Let $g(x)$ denote the number of $-1$ a sample $x$ has. We have for all $\{x\in X|g(x)=c\}$, they will have the same value of $\sum_x$ and thus their value on $f_1,...f_5$ will be the same. In this way, we can categorize samples into 6 cases where $g(x)=0, g(x)=1, ..., g(x)=5$. And we can form a matrix where each row represent the type of sample and column represent its value on $f_j$. Thus the matrix we form can be represented as $M_{i,j} = f_j({x|g(x)=i})$.  In the case of $k=5$, the matrix is \begin{equation}
        M =\begin{bmatrix}  
        11 & 9 & 7 & 5 & 3 & 1 \\
        9 & 7 & 5 & 3 & 1 & 0 \\
        7 & 5 & 3 & 1 & 0 & 0 \\
        5 & 3 & 1 & 0 & 0 & 0 \\
        3 & 1 & 0 & 0 & 0 & 0 \\
        1 & 0 & 0 & 0 & 0 & 0 \\
    \end{bmatrix}
    \end{equation}. And we want to find a weight $u$ on the matrix such that $Mu = y$, where $y_i$ represent the corresponding parity value of $x$ represented by a row. In our example, $y= [1, -1, 1, -1, 1, -1]$. Notice that the matrix is triangular and full rank, thus there is a unique $u$ that solve the system. Also, notice that the eigenvalues of the matrix are all 1. Thus the norm of $u$ can be bounded by $\|u\|_2 \leq \sqrt{k}$. 
\end{proof}

We are going to relax the condition showing that as long as the weight $w_i$ on a set of neuron is not too far from 1. Then there still exist a $u^*$ with small norm that solve the system.
\begin{lemma} \label{lemma:neuron_condition}
    Fix $k$, there is a set of $k+1$ ReLU functions in the form $f_j(x) = \relu(w_j^\top x_i + b_j)$ with $b_j = \{k+1, k-1, k-3, ..., -k + 3, -k +1 \}$ where if $w_{j,i}$ satisfy $|w_{j,i}-1| \leq \frac{1}{2k}$, then there is a set of weights $u_j$ with $\|u\|_2 \leq 2\sqrt{k}$ modeling the k-degree parity function by $f(x) = \chi(x) = \sum_{j=1}^{k+1} u_jf_j(x)$. 
\end{lemma}
\begin{proof}

From the proof of lemma \ref{lemma:1}, we see for the solution $u^*$ to exist, all we required is the function outputs on different cases of $x$ form a triangular matrix and the dependency of the upper bound of $\|u\|_2$ will be on the smallest entry of the diagonal of the matrix. We will show that given $|w_{j,i}-1| \leq \frac{1}{2k}$, for all $x$ \begin{enumerate}
        \item If $\sum x + b_j \geq 1$, then $w_j^\top x + b_j \geq 1/2$ 
        \item If $\sum x + b_j \leq -1$, then $w_j^\top x + b_k \leq -1/2$
    \end{enumerate}. Using this result and by the fact that our construction have for all $x$ and a fixed $j$ either $\sum x + b_j \geq 1$ or $\sum x + b_j \leq -1$, we can replace the function $f_j = \relu(\sum x + b_j)$ with $f_j' = \relu(w_j^\top x + b_j)$ such that the matrix is still triangular and all diagonal entries $M'_{j,j} \geq 1/2, \forall j \in [k]$. This implies the smallest eigenvalue of the matrix $\lambda_{\text{min}}(M) \geq 1/2$. And we have $\|y\|_2 \geq \lambda_{\text{min}}(M) \|u\|_2$, which implies $\|u\|_2 \leq \frac{\|y\|_2}{\lambda_{\text{min}}(M)} = 2\sqrt{k}$

We now prove the earlier claim. Note \begin{align*}
    &|w_j \top x + b_j - (\sum x + b_j)| \\
    &= |\sum_i (w_{j,i} - 1) x| \\
    &= | <w_j-1, x> | \\
    \leq& \|w_j-1\|_2^2 \|x\|_2^2 \\
    \leq& k^2(|w_j-1|)^2 \\
    \leq& \frac{1}{4}
\end{align*}
Thus If $\sum x + b_j \geq 1$, then we have $w_j \top x + b_j \geq 3/4 \geq 1/2$ and if $\sum x + b_j \leq -1$, we have $w_j \top x + b_j \leq -3/4 \leq -1/2$. 
\end{proof}

\subsubsection{On a set of neurons}
We first redo the proof of the hidden progress paper with $B.3$, given the population gradient gap between two indexes, we can learn the sparse parity function that is defined on the index (which are with higher absolute value of gradient) without any error with high probability. The difference is that we do thresholding and only requiring the population gradient to be accurate enough.

\begin{theorem}
    Suppose the population gradient gap between the spurious coordinates $[s]$ and core coordinates $[c]$ is $\Delta_{s-c}$  On a set of neurons in the configuration of \cref{setting:initalization_scheme}, if $-a_i\chi_{[s]\setminus j}(w)= \sign(\xi_{s-1}), \forall j\in[s]$ and $-a_i\chi_{[s] \cup j}(w)= \sign(\xi_{s+1}) \forall j\in[c]$. Suppose we take $m$ sample points,  and let the first step learning rate to be $\mu = \frac{1}{g_{s}}$ where $g_s$ is the spurious gradient we found in \cref{procedure:find_gradient}. We then do thresholding by comparing the gradient of each coordinate to the empirical mean of the absolute value of gradient applied on the whole hidden weight vector. We zero out weights on coordinates that have gradient small then the mean. Then with $m \geq \max\{\frac{2s^2\log(2s/\delta)}{\xi_{s-1}^2}, \frac{\log(2s/\delta)n^2}{2u^2\Delta_{s-c}^2}, \frac{\log(2s/\delta)n^2}{2c^2\Delta_{s-c}^2}\}$, the neuron satisfy the condition in \cref{lemma:neuron_condition} w.p $\geq 1-\delta$
\end{theorem}

\begin{proof}
    We denote a random sample point by $\bold{x_i}$ and its corresponding gradient imposed on the a neuron weights by $\bold{g_{j}^i} = [-ya_j I\{w^\top x + b > 0\}x_j]$. 
    
    We will use $g_s, g_c, g_u$ refer to the population gradient calculated in \cref{procedure:find_gradient}. We want bound the probability of the event 1). Only spurious weights are kept after thresholding. 2).on all spurious indexes $|\mu \frac{\sum_{k=1}^m \hat{g_{i,j}^{(k)}}}{m} - 1 | \leq \frac{1}{2c}$ by finding the required number of $m$.   That is for all weights, if the weight $j$ is on a spurious variable, it must satisfy at the same time \begin{equation*}
        |\mu \Hat{\bold{g_j}} - 1 | \leq \frac{1}{2s} \text{ and } \Hat{\bold{g_j}} -  \frac{\sum_{k=1}^n \Hat{\bold{g_k}}}{n} > 0
    \end{equation*} 
    and if the weight $j$ is on a core/independent variable, it must satisfy  \begin{equation*}
        \Hat{\bold{g_j}} -  \frac{\sum_{k=1}^n\Hat{\bold{g_k}}}{n} < 0
    \end{equation*}
where $\Hat{\bold{g_j}} =\frac{ \sum_{i=1}^m g_j^i}{m}$.
The first inequality can be written as, for $j\in [s]$\begin{equation*}
    |\mu \bold{g_j} - 1 | \leq \frac{1}{2s} \iff |\bold{g_j} - E[\bold{g_j}]| \leq \frac{g_s}{2s}
\end{equation*}
Take the expectation for the second formula we have for $j \in [c] \cup [u]$ \begin{equation*}
    \E[\bold{g_j} -  \frac{\sum_{k=1}^n\bold{g_k}}{n}] \leq  -\frac{(c+u)\Delta_{s-c}}{n}
\end{equation*}
We want bound the probability of bad event by $\delta$
we use the hoeffding bound and require, for $j\in [s]$\begin{align*}
    &P[\bold{g_j} - E[\bold{g_j}] \geq \frac{g_j}{2s}] \leq \frac{\delta}{n} \\
    &P[|\bold{g_j} -  \frac{\sum_{k=1}^n\bold{g_k}}{n} - \E[\bold{g_j} -  \frac{\sum_{k=1}^n\bold{g_k}}{n}]| \geq \frac{(c+u)\Delta_{s-c}}{n}] \leq \frac{\delta}{n} \\
\end{align*}
and for $j \in [u]$ \begin{equation*}
    P[|\bold{g_j} -  \frac{\sum_{k=1}^n\bold{g_k}}{n} - \E[\bold{g_j} -  \frac{\sum_{k=1}^n\bold{g_k}}{n}]| \geq \frac{s\Delta_{s-c}}{n}] \leq \frac{\delta}{n}
\end{equation*}
If all three is satisfied then by union bound we would have for a given set of neuron the condition \cref{lemma:neuron_condition} is not satified w.p less than $\delta$.
In terms, by hoffding bounds, we have when $m \geq \max\{\frac{2s^2\log(2s/\delta)}{g_s^2}, \frac{\log(2s/\delta)n^2}{2(c+u)^2\Delta_{s-c}^2}, \frac{\log(2s/\delta)n^2}{2s^2\Delta_{s-c}^2}\}$, the condition will be satisfied with error probability less than $\delta$
\end{proof}
\subsubsection{Number of neurons required}
Following the same argument in \cite{barak_hidden_2023}, we want have a set of neurons that satisfy B.3. 
\begin{theorem}
    Take $r \geq 2^s \log(2s/\delta)$ and $m \geq \max\{\frac{2s^2\log(2s/\delta)}{g_s^2}, \frac{\log(2s/\delta)n^2}{2(c+u)^2\Delta_{s-c}^2}, \frac{\log(2s/\delta)n^2}{2s^2\Delta_{s-c}^2}\}$, w.p $\geq 1-\delta$, after the first gradient step we will get a set of neurons satisfy \cref{lemma:neuron_condition}. 
\end{theorem}
\begin{proof}
    For some $w_i \sim \{\pm 1\}^n$, the probability that $w_i$ satisfy the condition 1). $-a_i\chi_{[s]\setminus j}(w_{i})= \sign(\xi_{s-1}), \forall j\in[s]$ 
    
    
    is $2^{-s}$. 2. Additionally, for some fixed $i' \in [s]$, the probability that $b_i = -s+i'$ is $\frac{1}{s}$. Therefore for some fixed $i\in[r/2]$ and $i' \in [s]$, with probability $\frac{1}{s2^{-s}}$, $b_i = b_{i+r/2} = -k + i'$ and the weight satisfy 1). Taking $r \geq 2^s \log(s/\delta)$, we get the probability that there is no $i\in[r/2]$ that satisfies the above condition for any fixed $i'$ is :
    \begin{equation*}
        \left(1-\frac{1}{n2^{n-1}}\right)^{r/2} \leq \exp(-\frac{r}{2n2^{n-1}}) \leq \frac{\delta}{s}
    \end{equation*}
    By union bound, w.p $\geq 1-\delta$, there exist a set of $s$ neurons satisfying the conditions of theorem 3. 
    
\end{proof}
\subsubsection{Stochastic gradient descent}
We use the following result on convergence of SGD (see \citep{shalev-shwartz_understanding_2014}). \begin{theorem}
    Let $M, \rho>0$. Fix $T$ and let $\mu=\frac{M}{\rho \sqrt{T}}$. Let $F$ be a convex function and $u^* \in \arg\min_{\|u\|_2 \leq M} f(u)$. Let $u^{(0)} = 0$ and for every $t$, let $v_t$ be some random variable s.t. $E[v_t|u^{(t)}] = \nabla_{u^{(t)}} F(u^{(t)})$ and let $u^{(t+1)} = u^{(t)} - \mu v^{(t)}$. Assume that $\|v_t\|_2 \leq \rho$ w.p. 1. Then, \begin{equation*}
        E[F(u^{(t)})] \leq F(u^*) + \frac{M\rho}{\sqrt{T}}
    \end{equation*}
\end{theorem}

\begin{theorem}(SGD on MLPs learns sparse parities)
Take $r \geq 2^s \log(2s/\delta),  \\ B \geq m \geq \max\{\frac{2s^2\log(2s/\delta)}{g_s^2}, \frac{\log(2s/\delta)n^2}{2(c+u)^2\Delta_{s-c}^2}, \frac{\log(2s/\delta)n^2}{2s^2\Delta_{s-c}^2}\}, T\geq \frac{37srn^2}{\epsilon^2}$, take the first gradient step with size $\mu = \frac{1}{g_s}$ then w.p $\geq 1-\delta$, by fixing the hidden layer weights and runs a SGD on the second layer weights $a$ with step size $\mu=\frac{2\sqrt{s}}{3\sqrt{rT}n}$, we can solve the parity task with error less than $\epsilon$
\end{theorem}
\begin{proof}
    We take the first gradient step and then fix the hidden layer. Let $F(u) = E_x[l(u^\top \relu(W^{(1)}x + b^{(1)}), y)]$. Thus $F$ is a convex function. For every $t$, denote \begin{equation*}
        v_t = \frac{1}{B}\sum^B_{i=1} \nabla_{u^{(t)}}l(f(x_{t,l};\theta_t), y) = \frac{1}{B}\nabla_{u^(t)}l((u^{(t)})^\top \relu(W^{(1)}x_{i,t} + b^{(1)}), y_{l,t}).
    \end{equation*}
    Note that by the condition, we have a set of good neurons after first gradient step w.p $> 1-\delta$ and there exists $\|u^*\|_2 \leq 2\sqrt{k}$ s.t $F(u^*)=0$ and for all $i,x$ it holds that $\|\relu(w_ix + b_i)\|_\infty \leq 2(n+1)$. Using this, we get \begin{equation*}
        \|v_2\|_2 \leq \frac{1}{B}\sum_{i=1}^B \|\relu(W^{(1)}x_{i,t} + b^{(1)}), y_{i,t} \leq 2\sqrt{r}(n+1)
    \end{equation*}
    Now we can apply the theorem of SGD with $M=2\sqrt{k}$ and $\rho=3\sqrt{r}n$ and get that w.p $>1-\delta$ over the initalization and the first step, it holds that \begin{align*}
        E[min_{t\in\{2,...,T\}} l(f(x;\theta_t),y)] &\leq E[\frac{1}{T-1}\sum^T_{t=2} l(f(x;\theta_t), y)] \\
        & = E[\frac{1}{T-1}\sum_{t=2}^T F(u^{(t)}] \leq 0 + \frac{6\sqrt{kr}n}{T-1} \leq \epsilon
    \end{align*}
\end{proof}

And finally because $m \leq O(\frac{\log(2s/\delta)n^2}{2(c+u)^2\Delta_{s-c}^2})$, we have $m \leq O(\frac{1}{\Delta_{s-c}^2}) = O(\frac{1}{\lambda n^{-s} - n^{-c}})$

\subsection{Dynamics after spurious feature has been learned}
\begin{lemma}
    Suppose a model only has access to the spurious coordinates $x_s$, then under cross entropy loss with distribution $\dlam$, the Bayes optimal model output $\log(\frac{\lambda}{1-\lambda})f_s(x_s)$
\end{lemma}
\begin{proof}\label{Appendix:proof:bayes_spurious}
    Let the bayes classifier be $g(x)$. The bayes classifier is defined as $\inf_{g(x)}E_{x,y}[l(g(x),y)]$. Here\begin{align*}
        \inf_{g(x)} L(g(x)) =&E_{x,y}[l(g(x_s),y)] \\
        =&E_{x_s}E_{y|x_s}[l(g(x_s)),y)] \\
        =&E_{x_s}[P_{f_s(x_s)=f_c(x_c)|x_s}l(g(x_s)) + P_{f_s(x_s)\neq f_c(x_c)|x_s}l(g(x_s))] \\
        =&E_{x_s}[\lambda (-\log(\phi(f_c(x_c)g(x_s)))) + (1-\lambda) (-\log(\phi(f_c(x_c)g(x_s))))] \\
        =&E_{x_s}[\lambda (-\log(\phi(f_s(x_s)g(x_s)))) + (1-\lambda) (-\log(\phi(-f_s(x_s)g(x_s))))] \\
    \end{align*}
    The function is convex and thus has a global optimum. We take the derivative w.r.t $g(x_s)$ and set it to zero. We use the property $\phi(-x) = 1-\phi(x)$
    \begin{align*}
        L'(g(x)) =& \lambda (\phi(f_s(x_s)g(x_s))-1) f_s(x_s) - (1-\lambda)(\phi(-f_s(x_s)g(x_s))-1)f_s(x_s)\\
        =& f_s(x_s) (\lambda (\phi(f_s(x_s)g(x_s))-1) - (1-\lambda)(\phi(-f_s(x_s)g(x_s))-1)) \\
        =& f_s(x_s) (\lambda (\phi(f_s(x_s)g(x_s))-1) + (1-\lambda)(\phi(f_s(x_s)g(x_s))) \\
        =& \phi(f_s(x_s)g(x_s)) - \lambda = 0
    \end{align*}
    Solving the equation, we get $g(x_s) = \log(\frac{\lambda}{1-\lambda})f_s(x_s)$.
\end{proof}

\begin{lemma} \label{Appendix:proof:slowed_down_ratio}
    Suppose a neural network can be decomposed to the sum of two models $h_s = \sum_{i\in S} a_i \relu(w_i x), h_c=\sum_{i\in C} a_i\relu(w_i x)$ which learns different features of the data. The learning process of $h_c$ then will be slowed down by $4\lambda (1-\lambda)$ when compared to the gradient where there is no spurious correlation. 
\end{lemma}
\begin{proof} 
    Let the model be $m(x) = h_s(x_s) + h(x)$. From Lemma \ref{lem:bayes_opt}, we have $\phi(f_s(x_s)h_s(x_s))=\lambda$. Suppose $h(x) = \sum_{i=1}^k b_i \sigma(w_i^\top x)$. Let us compute the population gradient $\nabla_{w_i} L_{\Dist{\lambda}}(m(x))$: \begin{align*}
        &\nabla_{w_i} L_{\Dist{\lambda}}(m(x))\\
        =& - \lambda\nabla_{w_i} E_{\Dist{same}}[\log(\phi(f_s(x_s)m(x))] - (1-\lambda)\nabla_{w_i} E_{\Dist{diff}}[\log(\phi(-f_s(x_s)m(x))]\\
        =& - \lambda E_{\Dist{same}}\left[(1 - \phi(f_s(x_s)h_s(x))) f_s(x_s) a_i \sigma'(w_i^\top x) x\right] + (1-\lambda) E_{\Dist{diff}}\left[(1 - \phi(-f_s(x_s)h_s(x))) f_s(x_s) a_i \sigma'(w_i^\top x) x\right]\\
        =& - \lambda ( 1- \lambda) E_{\Dist{same}}\left[f_s(x_s) a_i \sigma'(w_i^\top x) x\right] + \lambda (1-\lambda) E_{\Dist{diff}}\left[f_s(x_s) a_i \sigma'(w_i^\top x) x\right]\\
        =& - \lambda ( 1- \lambda)  E_{\Dist{unif}}\left[(f_s(x_s) + f_c(x_c)) a_i \sigma'(w_i^\top x) x\right] + \lambda ( 1- \lambda)  E_{\Dist{unif}}\left[(f_s(x_s) - f_c(x_c)) b_i \sigma'(w_i^\top x) x\right]\\
        =& - 2\lambda ( 1- \lambda)  E_{\Dist{unif}}\left[f_c(x_c)a_i \sigma'(w_i^\top x) x\right].
    \end{align*}

     We compare the gradient toward $h_(x)$ when $\lambda = 1/2$ and we have \begin{equation*}
         \frac{- 2\lambda ( 1- \lambda)  E_{\Dist{unif}}\left[f_c(x_c)a_i \sigma'(w_i^\top x) x\right]}{-  \frac{1}{2}  E_{\Dist{unif}}\left[f_c(x_c)a_i \sigma'(w_i^\top x) x\right]} = 4\lambda(1-\lambda)
     \end{equation*}
\end{proof}



\begin{lemma} \label{dead_spurious_neurons}
    Let $I_c$ denote core coordinates indexes and $I_s$ denote spurious coordinates indexes. Suppose a neural network can be decomposed into $h_s(x_s)+h_c(x_c)$. Further, suppose the core network gives the same prediction $\gamma_c f_c(x_c)$ on all $x_c$ i.e $h_c(x_c)=\beta, \forall x_c$. If $\sum_{i \in I_c}|w_i| < |w_j|$ for all $j \in I_s$. Then the gradient on core coordinates of the spurious neuron will be 0.
\end{lemma}

\begin{proof}
     We have for any weight on a core coordinate $c \in I_c$, $\nabla_{w_c} L_{\Dist{\lambda}}(m(x)) = - \lambda\nabla_{w_c} E_{\Dist{same}}[\log(\phi(f_s(x_s)m(x))] - (1-\lambda)\nabla_{w_c} E_{\Dist{diff}}[\log(\phi(-f_s(x_s)m(x))]$. We will show  $\nabla_{w_c}E_{\Dist{same}}[\log(\phi(f_s(x_s)m(x))] = 0$ and the same can be shown in a similar manner for $\Dist{diff}$. \begin{align*}
         &\nabla_{w_c}E_{\Dist{same}}[\log(\phi(f_s(x_s)m(x))]\\
         =& E_{\Dist{same}}[(1-\phi(f_s(x_s)h_s(x_s)+\gamma_c)f_s(x_s)a\relu'(w^\top x)x_c] \\
         =& 2aE_{\Dist{unif}}[(1-\phi(f_s(x_s)h_s(x_s)+\gamma_c)f_s(x_s) \frac{f_s(x_s)f_c(x_c)+1}{2}x_c] \\
         =& aE_{\Dist{unif}}[(1-\phi(f_s(x_s)h_s(x_s)+\gamma_c)(f_s(x_s)+f_c(x_c))\relu'(w_s^\top x_s)x_c] \\
         =& aE_{\Dist{unif}}[(1-\phi(f_s(x_s)h_s(x_s)+\gamma_c)\chi_{I_c\setminus \{c\}}\relu'(w_s^\top x_s)] +
         aE_{\Dist{unif}}[(1-\phi(f_s(x_s)h_s(x_s)+\gamma_c)f_s(x_s)\relu'(w_s^\top x_s)x_c] \\
         =& 0 + 0 
     \end{align*}

\end{proof}
\section{Additional Experiments}
\label{Appendix:more experiments}
\subsection{Detailed Experiments configuration}
\label{Appendix:Exp_configs}
In this section, we conduct a more detailed experiment report and discussion of the results and claims presented in the main paper. We first show the training procedure we adopt. 

\paragraph{Model and Default Hyper-Parameters} 
All training experiments were conducted using PyTorch\citep{paszke_pytorch_2019}. While the majority of networks evaluated in our primary empirical findings are relatively compact, we trained a substantial number of models to validate the breadth of the "robust space" outcomes. These experiments utilized NVIDIA T4 and Quadro RTX 8000 GPUs, cumulatively consuming around 2,500 GPU hours.

Neural networks were initialized using a uniform distribution (as by the default setting of pytorch). For all the boolean experiments, without further specification, we use a 2-layer, 100 neurons NN. For the Domino and Waterbird datasets, we adopt the ResNet-50 and add an additional linear layer on top of it based on the number of class label. The model is trained either from scratch with randomly initialized weights or finetuned with \verb|IMAGENET1K_V1| pre-trained weights. The default optimization parameters is listed in \cref{table:default_optimization_parameters}. We adopt the parameters based on \cite{labonte_towards_2023, sagawa_distributionally_2020, idrissi_simple_2022, kirichenko_last_2023, izmailov_feature_2022} with minor modifications to ensure consistency across different experiments. Notably, for different boolean datasets and domino datasets we control the number of gradient updates of one epoch to be the same. This ensure us to compare them in a consistent manner. We provide further experiments to check how different hyper-parameters choices affect learning dynamics on different datasets \cref{fig:depth}, \cref{fig:waterbirds} and \cref{fig:weight-decay}. 

\begin{table}[htbp]
\centering
\caption{Default Training Hyper-Parameters Across Datasets}
\begin{tabular}{@{}llllllll@{}}
\toprule
Dataset & $\eta$ & $\beta$ & Optimizer & Weight Decay & $B$ & Reweight Method& Sample per Epoch \\ \midrule
boolean & 0.0001 & 0.5 & SGD & 0 & 64 & N/A & 10000 \\
domino & 0.001 & 0.9 & Adam & 0 & 64 & N/A & 5000 \\
waterbirds & 0.003 & 0.5 & SGD & 0.0001 & 32 & reweight & 500 \\
celebA & 0.001 & 0.5 & SGD & 0.0001 & 100 & resample & dataset size \\
cmnist & 0.003 & 0.5 & SGD & 0.0001 & 32 & N/A & dataset size \\
multinli & 0.00001 & N/A & AdamW & 0.0001 & 16 & resample & dataset size \\
civilcomment & 0.00001 & N/A & AdamW & 0.0001 & 16 & resample & dataset size \\ \bottomrule
\end{tabular}
\caption{
The optimization hyperparameters for various datasets are outlined in the table with $\eta$: learning rate, $\beta$: momentum, $B$: batch size. There are three possible methods for "Label Class Reweight." For class balanced dataset, such as the Domino, CMnist, and boolean datasets, no class balancing technique is necessary. In cases of class imbalance, we employ either resampling, which ensures each batch contains an equal number of samples from each class, or reweighting, which assigns greater importance or weight to samples from the minority class, as discussed in \citep{idrissi_simple_2022}.  Should there be any deviations from these default parameters in specific experiments, such changes will be explicitly noted.}
\label{table:default_optimization_parameters}
\end{table}

\paragraph{Training and Evaluation Procedure}
As specified in the main paper, we mainly use correlation and decoded correlation to assess the state of the model. For datasets that has more than two labels or spurious classes (MultiNli, CivilComments, CMnist), we use a generalized metrics of correlation $\mathsf{corr}(f_c,h)= P_x[f_c(x) = h(x)] - P_x[f_c(x) \neq h(x)]$ where $f_c$ is the core/ground truth label and $h(x)$ represent the prediction made by the trained model. 

The decoded correlation is calculated in the following procedure:
\begin{enumerate}
    \item At the end of each epoch, we select $\min(2000, \lceil n/2 \rceil)$ samples from the group balanced validation dataset and use the model to output their embeddings, where $n$ is the size of the group balanced dataset.
    \item We then fit a logistic regression model on the output embeddings with default hyperparameters provided by the Scikit-learn package \citep{pedregosa_scikit-learn_2011}.
    \item Finally, we take another $\min(2000, \lceil n/2 \rceil)$ samples from the validation dataset and output their embeddings. We then evaluate the correlation of the trained logistic regression model on the output embeddings. We note that the default hyperparameters in most cases could provide the near optimal correlation score, and adjusting the hyperparameters like the regularization term or strength only yields marginal improvement.
\end{enumerate}

\begin{figure} \label{fig:real_world_dynamics}
    \centering
    \includegraphics[width=1\textwidth]{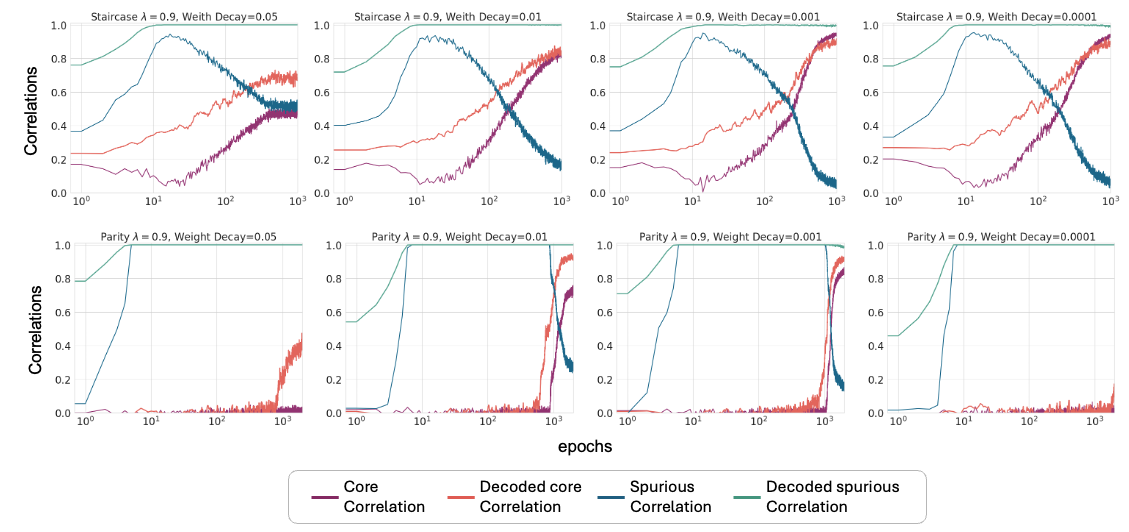}
    \caption{Learning dynamics on Spurious Boolean Dataset under different weight decay/l2 regularization values.}
    \label{fig:weight-decay}
\end{figure}

\begin{figure}
    \centering
    \includegraphics[width=1\textwidth]{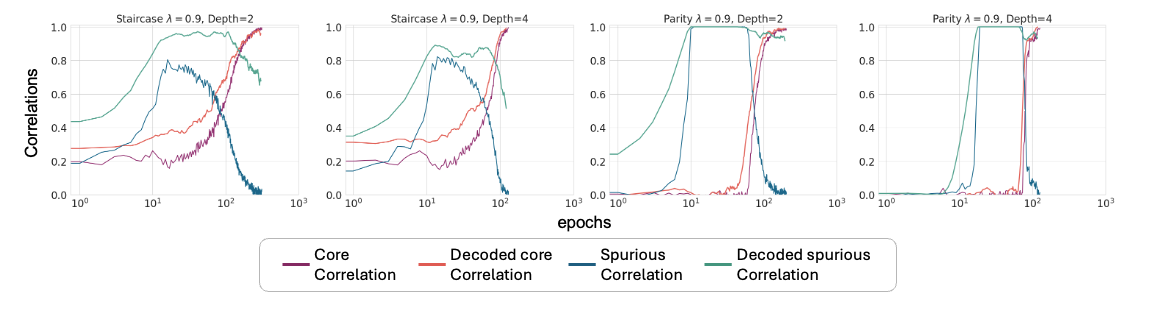}
    \caption{Learning dynamics on Spurious Boolean Dataset under different model depths.}
    \label{fig:depth}
\end{figure}

\begin{figure}
    \centering
    \includegraphics[width=1\linewidth]{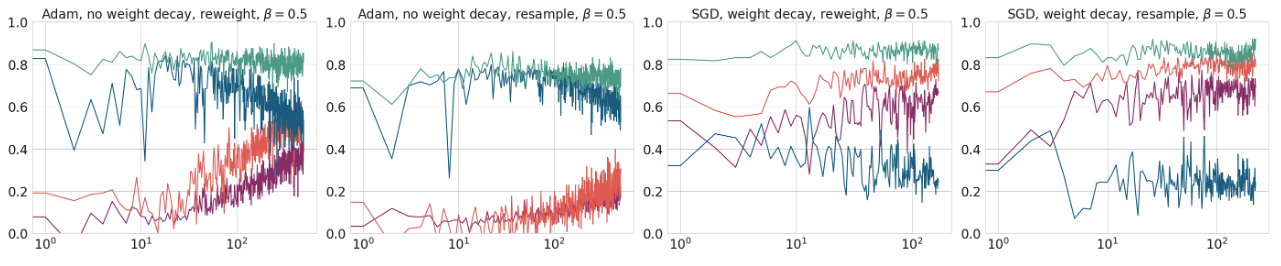}
    \caption{Learning dynamics on Waterbirds under different hyperparameters. $\lambda=0.95$. Model is pretrained ResNet. lr=0.001. Weight Decay is $0.0001$ or $0$, $\beta$ is Momentum.}
    \label{fig:waterbirds}
\end{figure}

\subsection{Additional Experiments on the interplay between Confounder Strength and Complexity}
We divide this section into two sections to provide a comprehensive review of the influence of the two factors, complexity and confounder strength on learning either on a online setting learning or a finite dataset.
\subsubsection{Complexity}

\paragraph{Parity}
\cref{fig:parity_online_complexity}, \cref{fig:parity_finite_complexity} show the learning dynamics of parity functions under different $deg(f_s)$. Note that the variance between repeated experiments is significant when $\lambda$ and the complexity of the spurious function ($\text{deg}(f_s)$) are both high. In the context of learning parity with finite datasets, numerous runs converging to a low core correlation value. For learning under finite dataset, it is worth highlighting that the end performance of the network is heavily influenced by the randomness of initialization. Note here the total length of the feature vector is fixed to $20$ so the computational complexity in learning core parity function stay fixed if $\lambda=0.5$ for each case.
\begin{figure}
    \begin{minipage}{\textwidth}
        \centering
        \includegraphics[width=\linewidth]{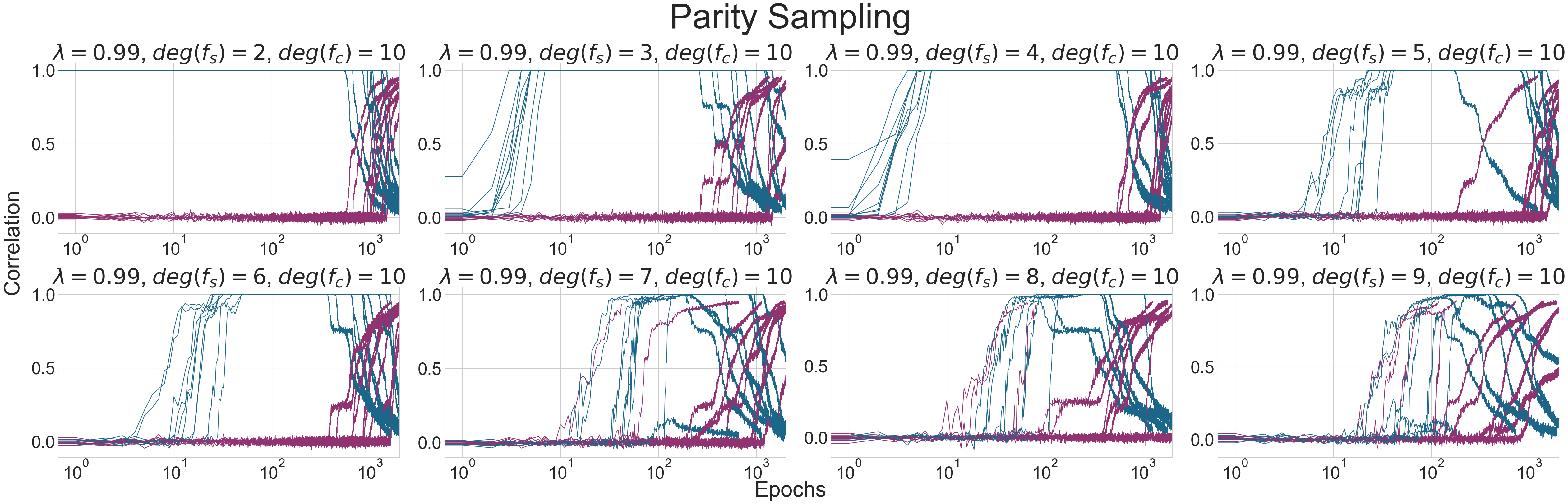}
        \label{fig:lambda_0.99a}
    \end{minipage}
    \hfill
    \begin{minipage}{\textwidth}
        \centering
        \includegraphics[width=\linewidth]{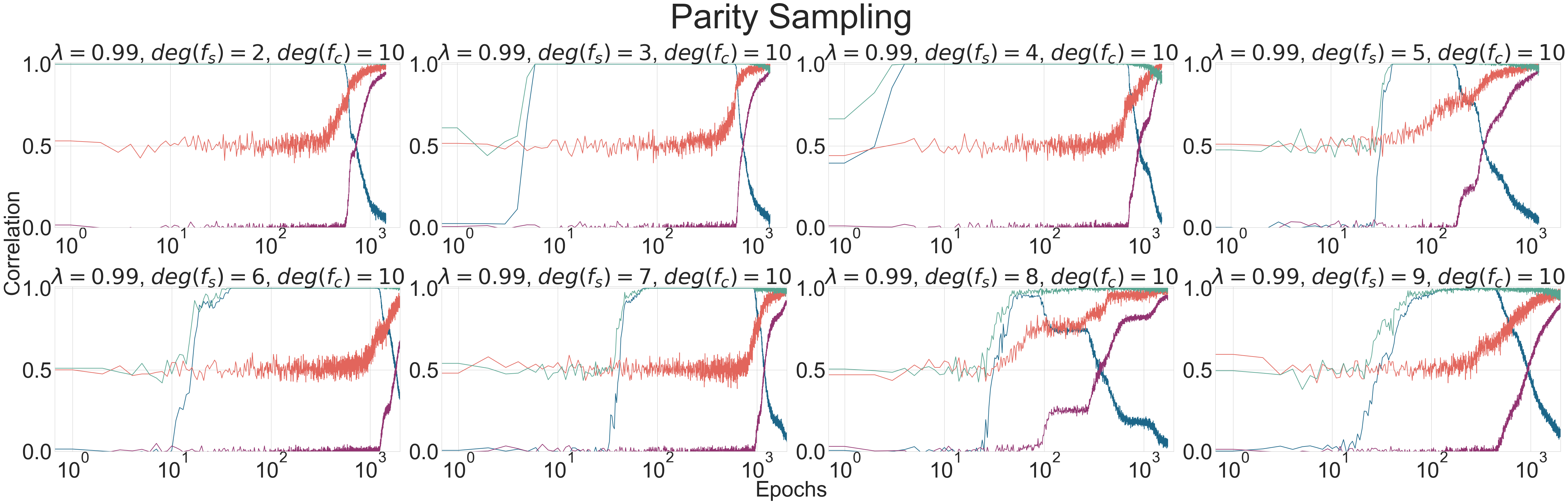}
        \label{fig:lambda_0.99b}
    \end{minipage}
    \begin{minipage}{\textwidth}
        \centering
        \includegraphics[width=0.5\linewidth]{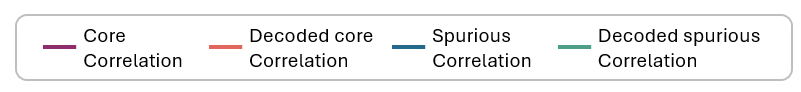}
    \end{minipage}
    \caption{\textbf{Online Parity Learning:} Upper: repeated experiments. Bottom: single experiment taken from the repeated experiments.}
    \label{fig:parity_online_complexity}
\end{figure}

\begin{figure}
    \centering
    \begin{minipage}{\textwidth}
        \centering
        \includegraphics[width=\linewidth]{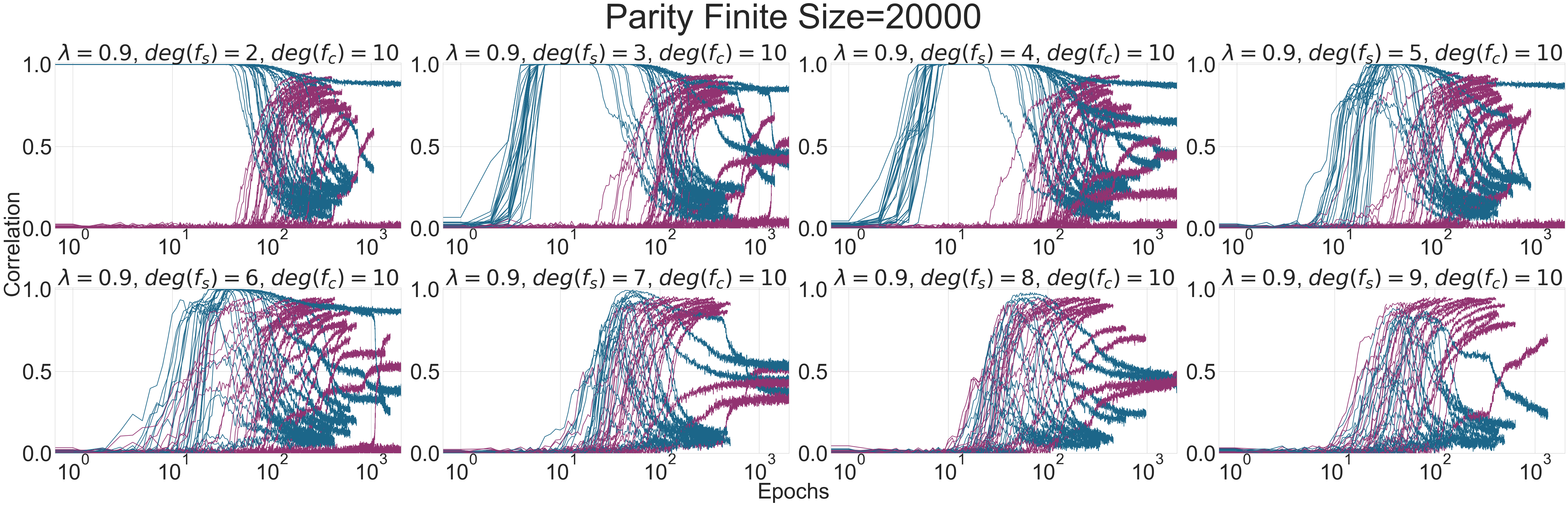}
        \label{fig:lambda_0.90a}
    \end{minipage}%
    \hfill
    \begin{minipage}{\textwidth}
        \centering
        \includegraphics[width=\linewidth]{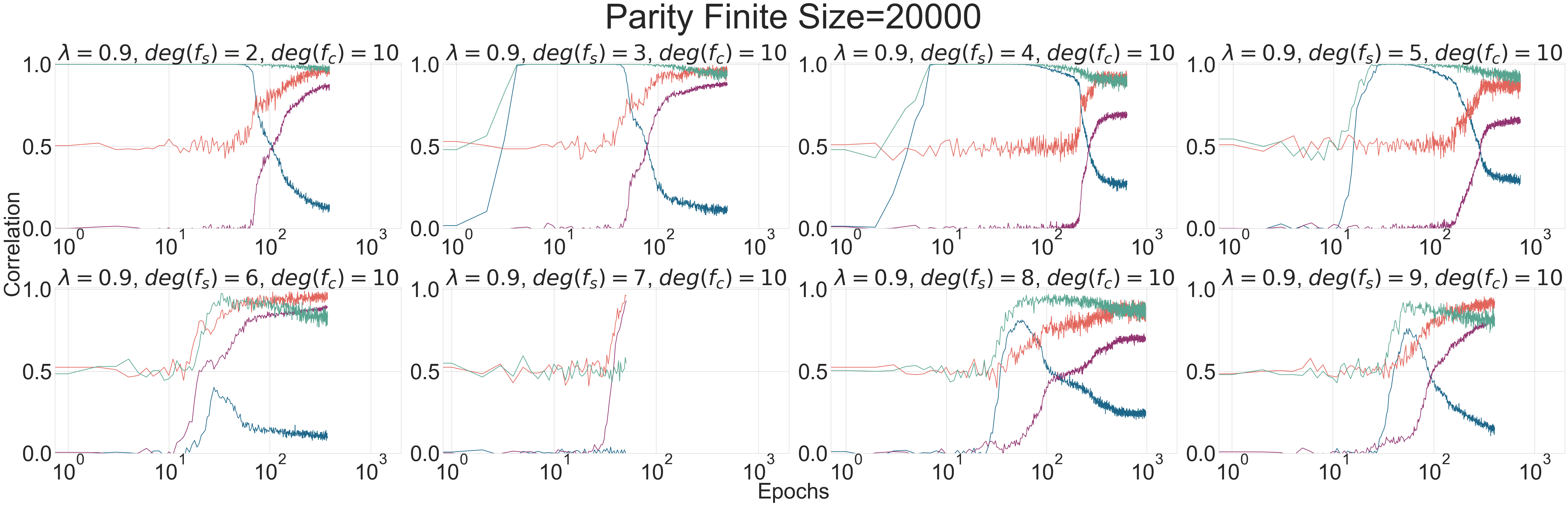}
        \label{fig:lambda_0.90b}
    \end{minipage}
    \begin{minipage}{\textwidth}
        \centering
        \includegraphics[width=0.5\linewidth]{icml/appendix/legend.png}
    \end{minipage}
    \caption{\textbf{Finite Parity Learning with 20000 Sampled Points.} Upper: repeated experiments Bottom: single experiment}
    \label{fig:parity_finite_complexity}
\end{figure}

\paragraph{Staircase} \label{appendix:fig_finite_staircase}
Refer to \cref{fig:staircase_online_complexity} and \cref{fig:staircase_finite_complexity}. In the case of the staircase task, the influence of simpler spurious features on convergence slowdown becomes more obvious. Different from the parity cases, the learning dynamics remain consistently stable across repeated runs in the staircase task. 
\begin{figure}
    \centering
    \begin{minipage}{\textwidth}
        \centering
        \includegraphics[width=\linewidth]{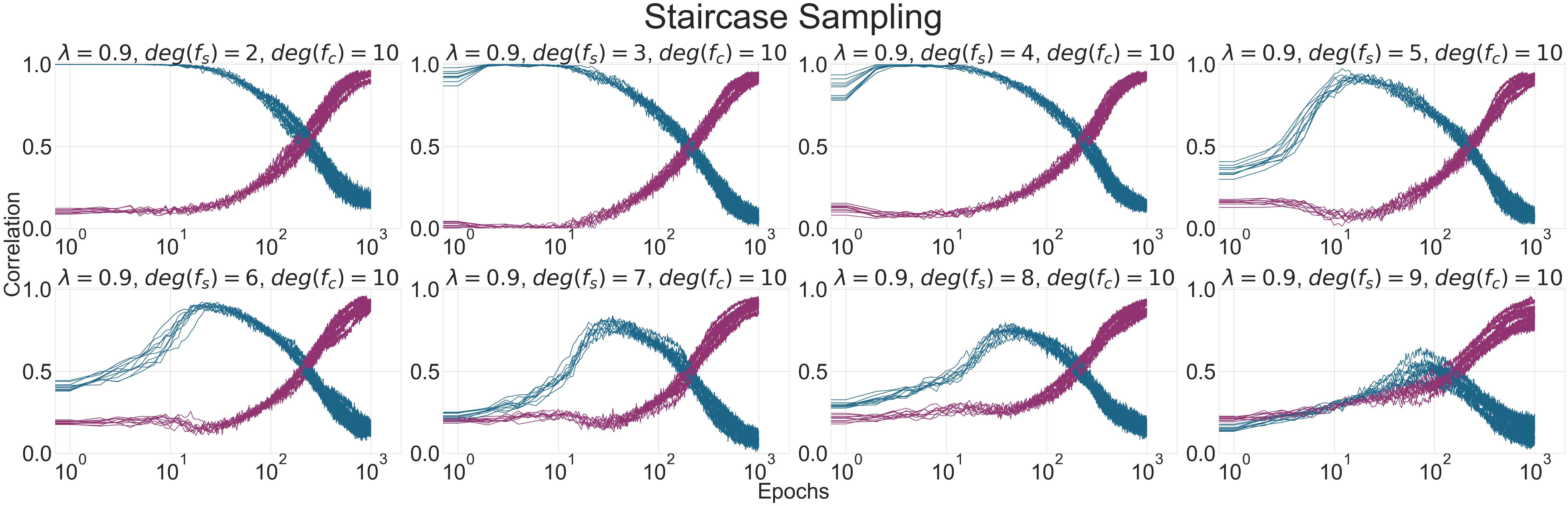}
        \label{fig:lambda_0.90_staircase_a}
    \end{minipage}%
    \hfill
    \begin{minipage}{\textwidth}
        \centering
        \includegraphics[width=\linewidth]{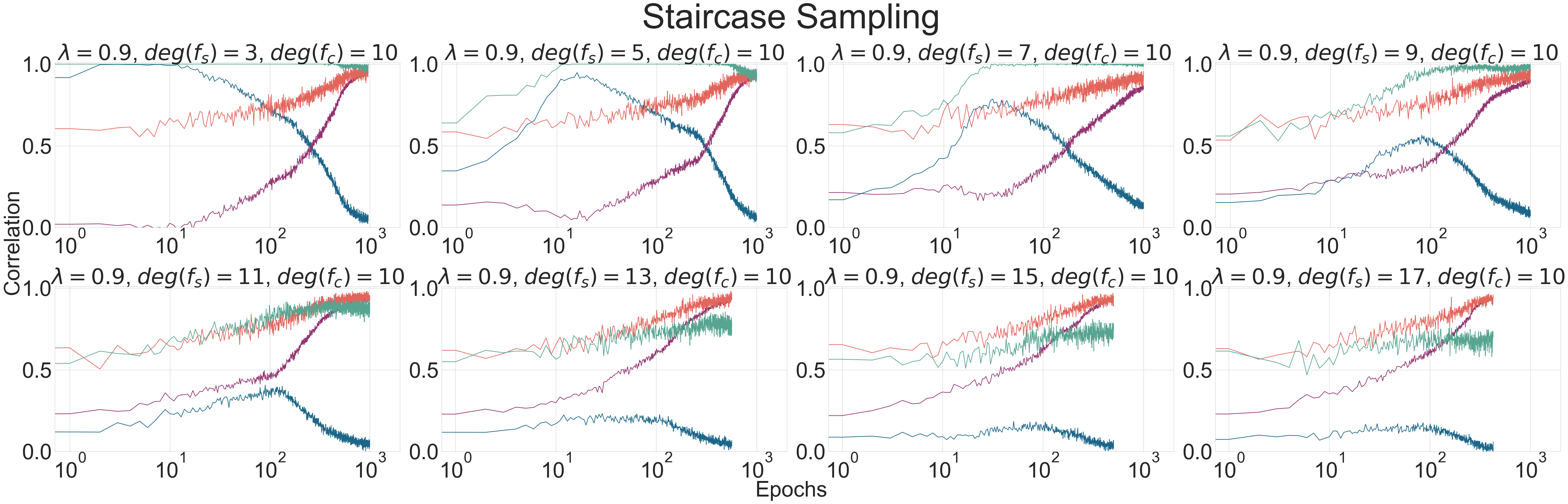}
        \label{fig:lambda_0.90_staircase_b}
    \end{minipage}
    \begin{minipage}{\textwidth}
        \centering
        \includegraphics[width=0.5\linewidth]{icml/appendix/legend.png}
    \end{minipage}
    \caption{\textbf{Online Staircase:} Upper: repeated experiments. Bottom: single experiment}
    \label{fig:staircase_online_complexity}
\end{figure}

\begin{figure}
    \centering
    \includegraphics[width=1\textwidth]{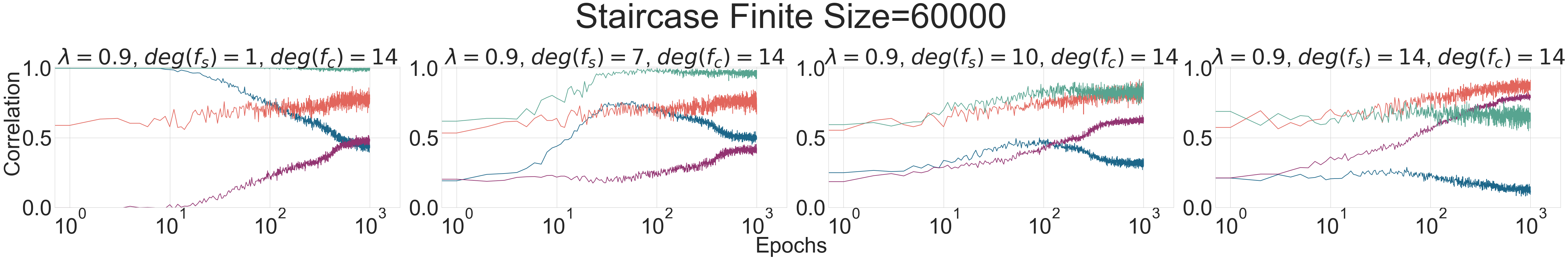}
    \caption{\textbf{Finite Staircase with 60000 Sampled Points:} $\lambda = 0.90$}
    \label{fig:staircase_finite_complexity}
\end{figure}

\paragraph{Domino}
See \cref{fig:domino_complexity_dynamics}. We adhere to the convention of employing three image datasets as spurious features: MNIST-01, MNIST-79, and Fashion dress-coat, arranged in ascending order of difficulty with CIFAR-truck-automobile as the core feature \cite{izmailov_feature_2022, kirichenko_last_2023}. It is noteworthy that the semi-real datasets including the domino datasets utilized in spurious correlation research are inherently noisy, meaning that the model cannot learn the core feature perfectly or achieve 0 generalization error, as highlighted in \citep{kirichenko_last_2023}. In fact, we see the the core correlation of the model is well below 0.9. Furthermore, these datasets are limited in size, with only 10,000 images available for CIFAR-truck-automobile.

Previous studies have primarily focused on utilizing pretrained models to learn the spurious task. However, such an approach can obscure our understanding of feature learning dynamics, as pretrained models often achieve exceptionally high decoded core correlations from the outset, as noted in \citep{joshi_towards_2023}. To better show this point, we present the learning dynamics for both pretrained and randomly initialized weights. We also found that \textbf{pretrained models are more robust to spurious features at different complexities}.

\begin{figure}
    \centering
    \begin{minipage}{\textwidth}
        \centering
        \includegraphics[width=\linewidth]{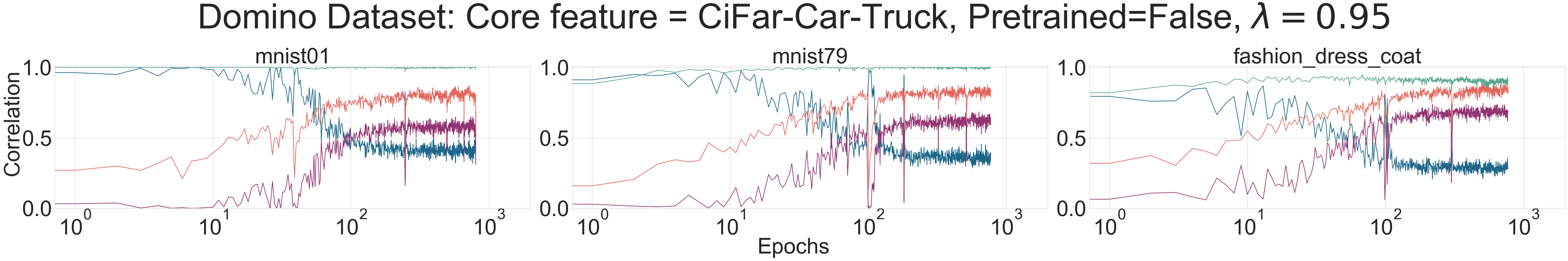}
        \label{fig:domino_no_pretrained}
    \end{minipage}%
    \hfill
    \begin{minipage}{\textwidth}
        \centering
        \includegraphics[width=\linewidth]{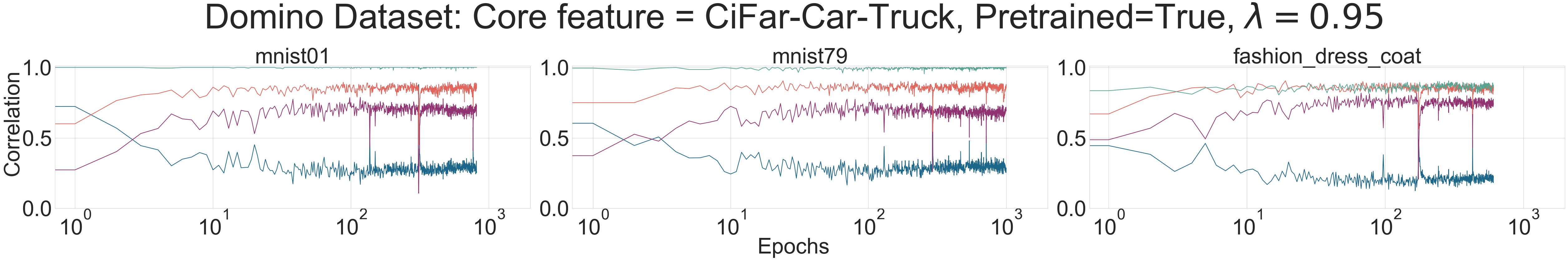}
        \label{fig:domino_pretrained}
    \end{minipage}
    \begin{minipage}{\textwidth}
        \centering
        \includegraphics[width=0.5\linewidth]{icml/appendix/legend.png}
    \end{minipage}
    \caption{\textbf{Learning dynamics of Domino dataset.} The spurious dataset become harder from left to right. The subplot title shows the spurious task. mnist01: Classification task of handwritten digits images of 0 and 1 taken from the MNIST dataset. mnist79: Classification task of handwritten digits images of 7 and 9 taken from the MNIST dataset. fashion dress coat: Classification task of dress and coat images taken from the FashionMnist dataset. The core task is classification of Truck and AutoMobile images taken from the CiFar dataset. The plots shows similar to the staircase dataset, harder spurious feature has less influence of the end performance of the model. Note confounder strength $\lambda=0.95$ is fixed.}
    \label{fig:domino_complexity_dynamics}
\end{figure}

\subsubsection{Confounder Strength} \label{Appendix:confounder_strength}
The impact of confounder strength on learning is more straight forward than the complexity. As confounder strength increases, the number of epochs needed for convergence also rises significantly. Notably, learning remains relatively insensitive to confounder strength until it reaches a threshold of 0.8, at which point we observe a notable increase in training epochs.The information of spurious feature i.e how well the spurious feature is learned and than memorized depends heavily on the confounder strength.
\paragraph{Parity}
For parity functions (\cref{fig:parity_online_cs}), we see when confounder strength surpass 0.9, it converge much slower after the phase transition when compared to the experiment with lower $\lambda$. The slower convergence reflect on learning under finite dataset where the end performance of the model is significantly impaired (\cref{fig:parity_finite_cs}).
\begin{figure}
    \centering
    \begin{minipage}{\textwidth}
        \centering
        \includegraphics[width=\linewidth]{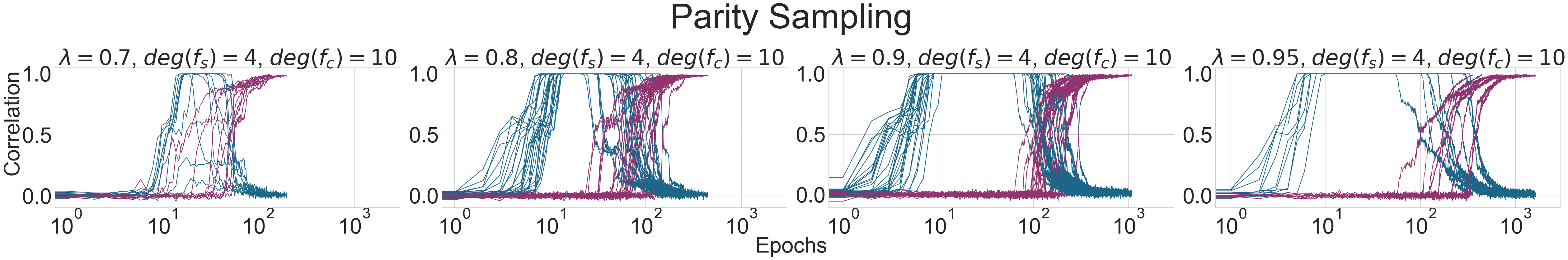}
        \label{fig:deg_fs_4_online_a}
    \end{minipage}%
    \hfill
    \begin{minipage}{\textwidth}
        \centering
        \includegraphics[width=\linewidth]{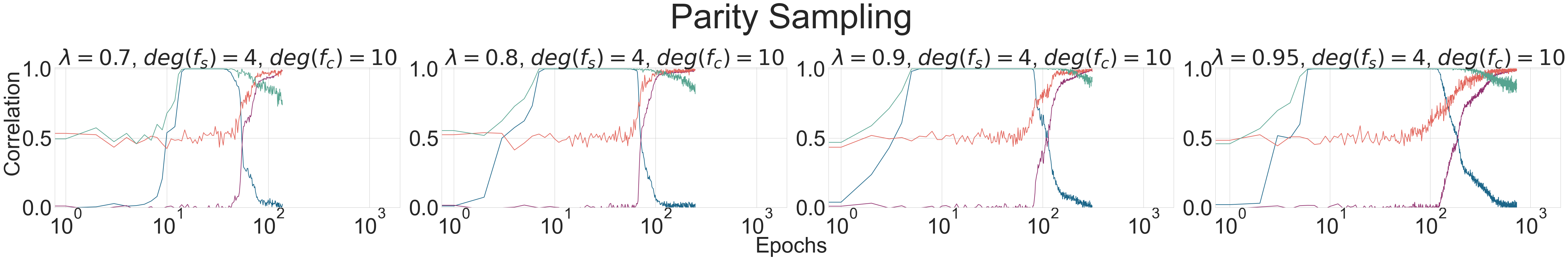}
        \label{fig:deg_fs_4_online_b}
    \end{minipage}
    \begin{minipage}{\textwidth}
        \centering
        \includegraphics[width=0.5\linewidth]{icml/appendix/legend.png}
    \end{minipage}
    \caption{\textbf{Online Parity}: Upper: repeated experiments Bottom: single experiment}
    \label{fig:parity_online_cs}
\end{figure}

\begin{figure}
    \centering
    \begin{minipage}{\textwidth}
        \centering
        \includegraphics[width=\linewidth]{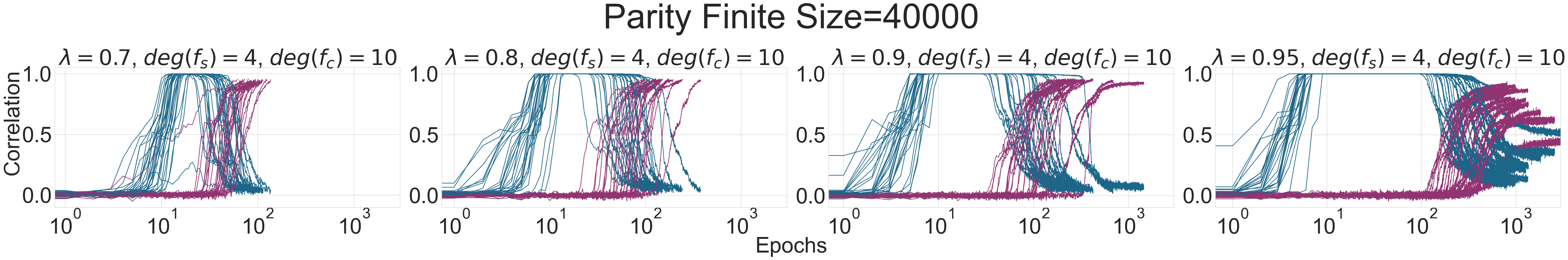}
        \label{fig:deg_fs_4_finite_a}
    \end{minipage}%
    \hfill
    \begin{minipage}{\textwidth}
        \centering
        \includegraphics[width=\linewidth]{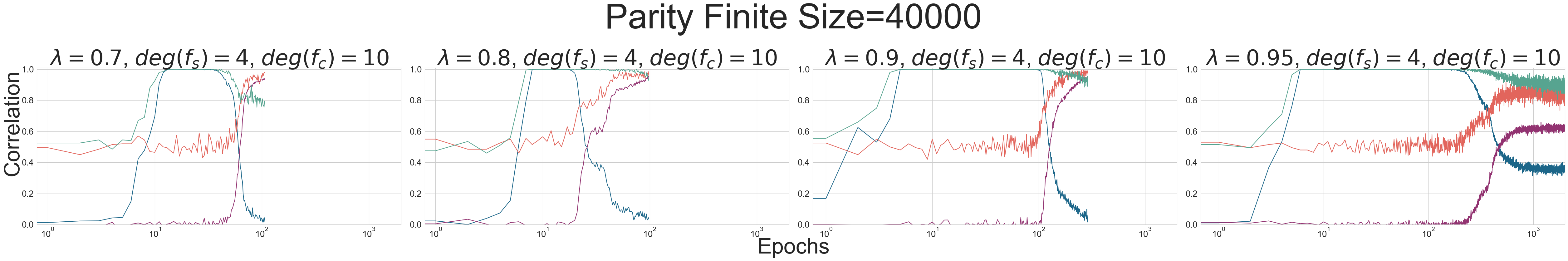}
        \label{fig:deg_fs_4_finite_b}
    \end{minipage}
    \begin{minipage}{\textwidth}
        \centering
        \includegraphics[width=0.5\linewidth]{icml/appendix/legend.png}
    \end{minipage}
    \caption{\textbf{Finite Parity with 40000 Sampled Points}: Upper: repeated experiments. Bottom: single experiment}
    \label{fig:parity_finite_cs}
\end{figure}

\paragraph{Staircase}
At higher confounder strength, the model has higher correlation to the spurious, simpler staircase function at the early stage of learning. This imply the spurious staircase function is learned and memorized better by the model. We see higher $\lambda$ cause harm to the end performance under finite dataset just as parity (\cref{fig:staircase_cs}).
\begin{figure}
    \centering
    \begin{minipage}{\textwidth}
        \centering
        \includegraphics[width=\linewidth]{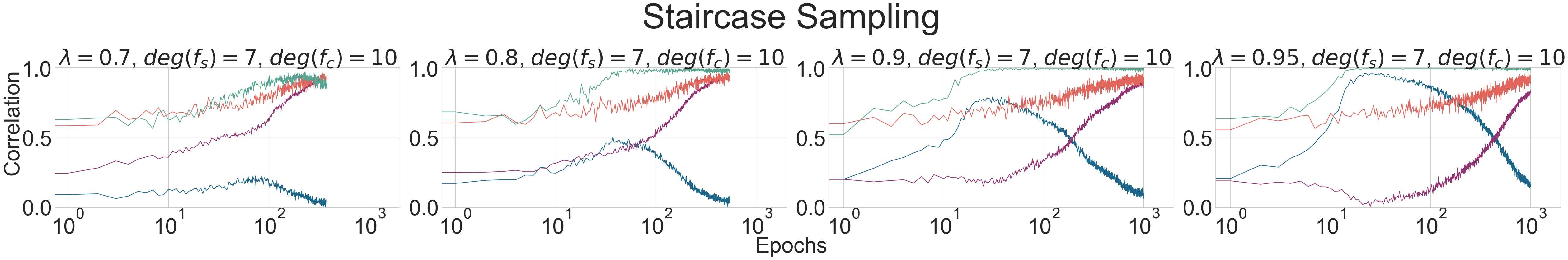}
        \label{fig:online_deg_fs_4}
    \end{minipage}
    \hfill
    \begin{minipage}{\textwidth}
        \centering
        \includegraphics[width=\linewidth]{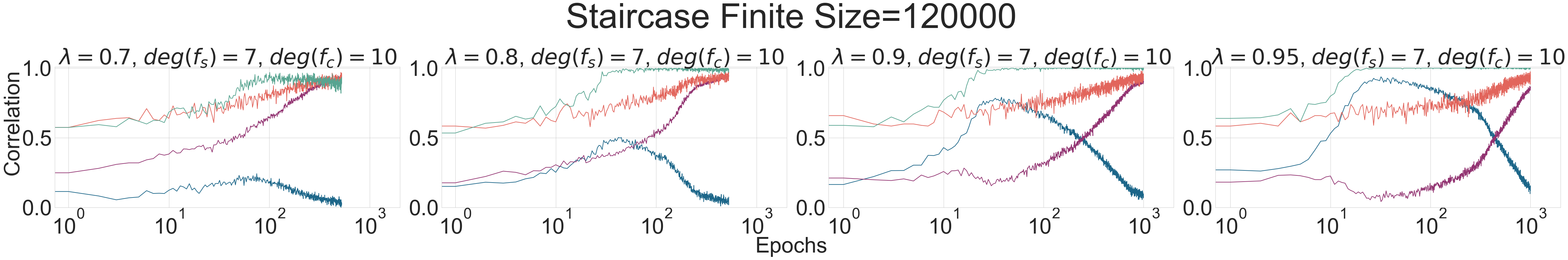}
        \label{fig:finite_120k_deg_fs_4}
    \end{minipage}
    \begin{minipage}{\textwidth}
        \centering
        \includegraphics[width=0.5\linewidth]{icml/appendix/legend.png}
    \end{minipage}
    \caption{\textbf{Staircase.} Upper: Learning dynamics under sampling. Bottom: Learning dynamics under finite dataset. }
    \label{fig:staircase_cs}
\end{figure}

\paragraph{Domino-Image, WaterBirds} (\cref{fig:domino_cs,fig:waterbirds_cs}) \label{Appendix:real_dataset}
Surprisingly, our observations indicate that the pretrained model exhibits not only insensitivity to spurious features across a spectrum of complexities but also a remarkable resistance to higher $\lambda$ values. Additionally, when compared to the initialization with random weights, models with pretrained weights consistently maintain low spurious correlations throughout the training process.

Regarding the waterbirds dataset, it is noteworthy that initialization with random weights fails to learn the core feature entirely, as reported in \citep{kirichenko_last_2023,joshi_towards_2023}.
\begin{figure}
    \centering
    \subfigure[Domino-Image: Pretrained Weights]{
        \includegraphics[width=1\textwidth]{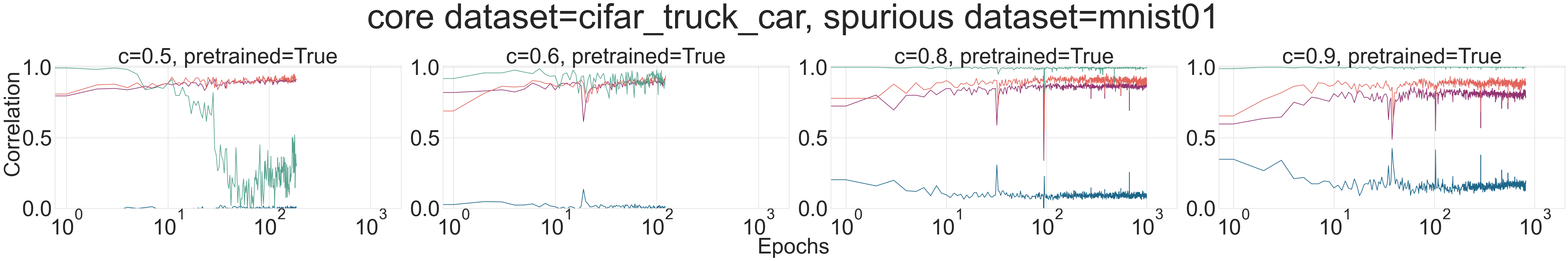}
    }
    \subfigure[Domino-Image: Random Weights]{
        \includegraphics[width=1\textwidth]{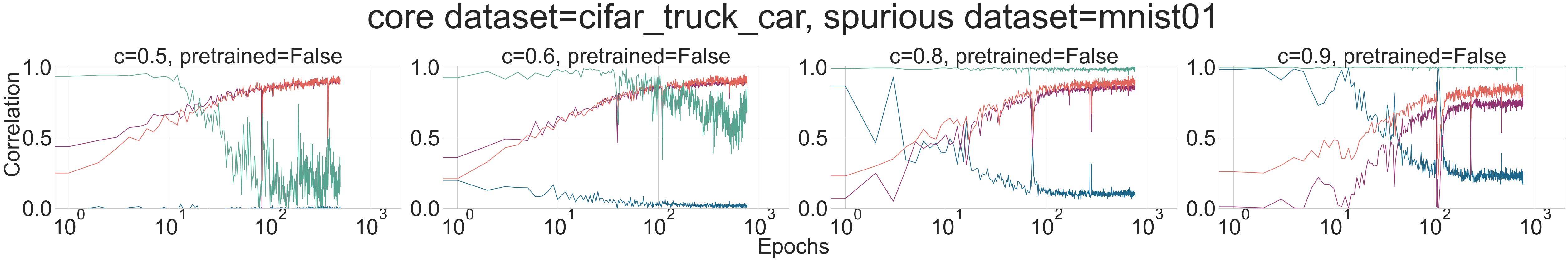}
    }
    \begin{minipage}{\textwidth}
        \centering
        \includegraphics[width=0.5\linewidth]{icml/appendix/legend.png}
    \end{minipage}
    \caption{\textbf{Domino-Image.} The plot shows pretrained model is more robust to the existence of a spurious feature across varying counfounder strength.}
    \label{fig:domino_cs}
\end{figure}

\begin{figure}
    \centering
    \includegraphics[width=1\textwidth]{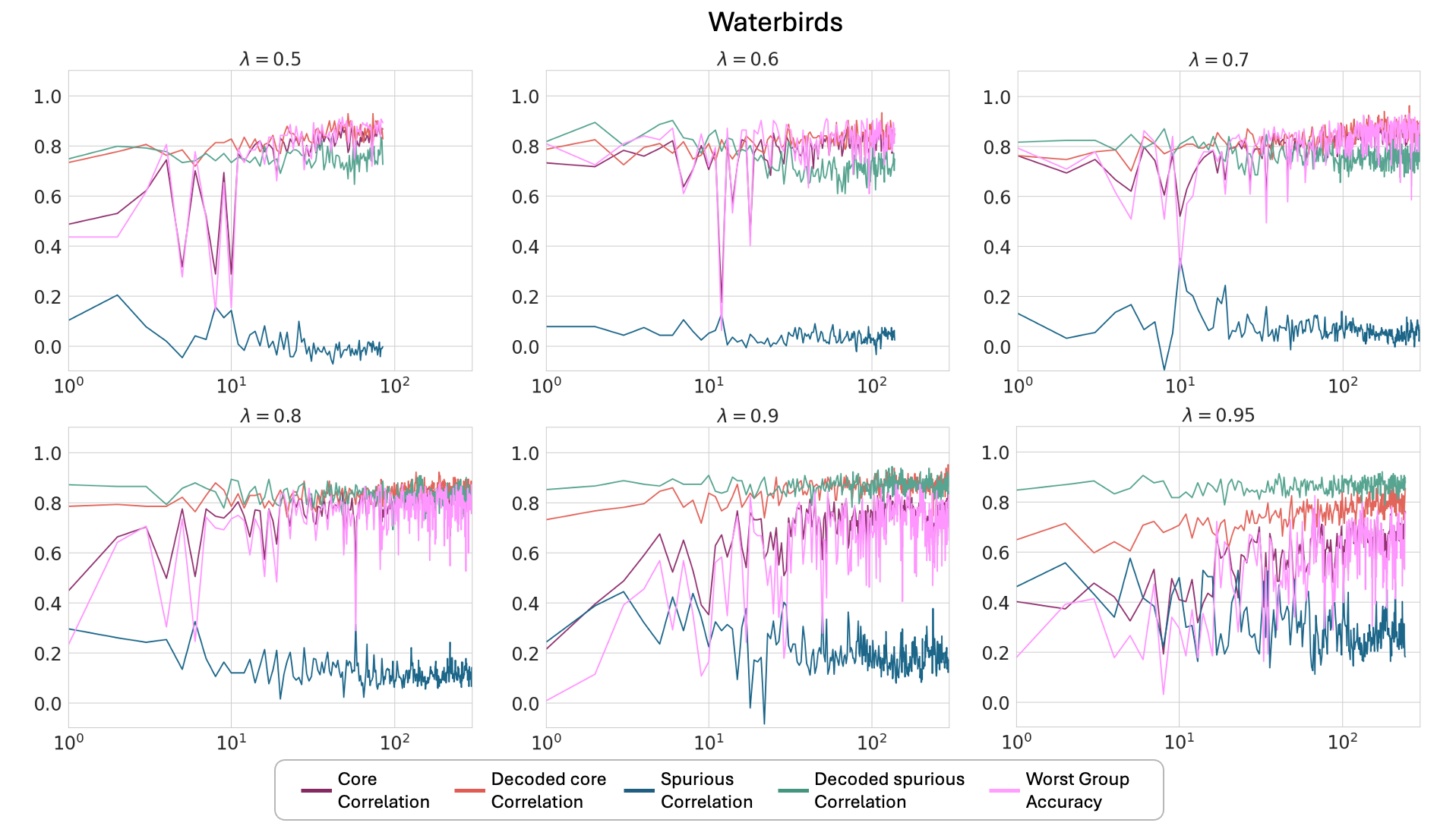}
    \caption{\textbf{Learning Dynamics on Waterbirds}}
    \label{fig:waterbirds_cs}
\end{figure}

\subsection{Core and Spurious subnetwork}
\subsubsection{More neuron plots on a 2 layer NN}
 \label{Appendix:neurons}
Refer to \cref{fig:parity_neuron}, \cref{fig:staircase_neuron}. We show the dynamics of a random batch of spurious neurons and core neurons for both the parity and staircase spurious learning task. It can be seen that spurious neurons have higher weights on spurious coordinates throughout training. And core neurons which has significant weights on the core coordinates are specifically the neurons which does not have spurious weight spike at the start when the spurious feature is learned. 

\begin{figure}
    \centering
    \subfigure[Parity: Core Neurons at $\lambda=0.95$]{
        \includegraphics[width=1\textwidth]{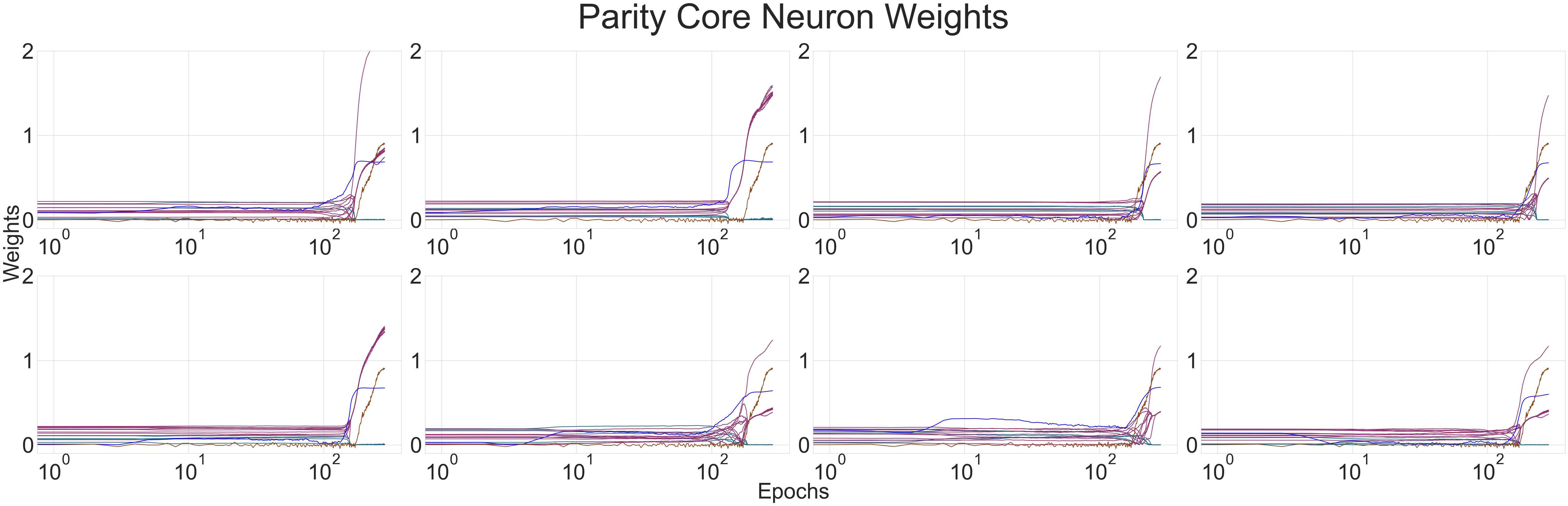}
    }
    \subfigure[Parity: Spurious Neurons at $\lambda=0.95$]{
        \includegraphics[width=1\textwidth]{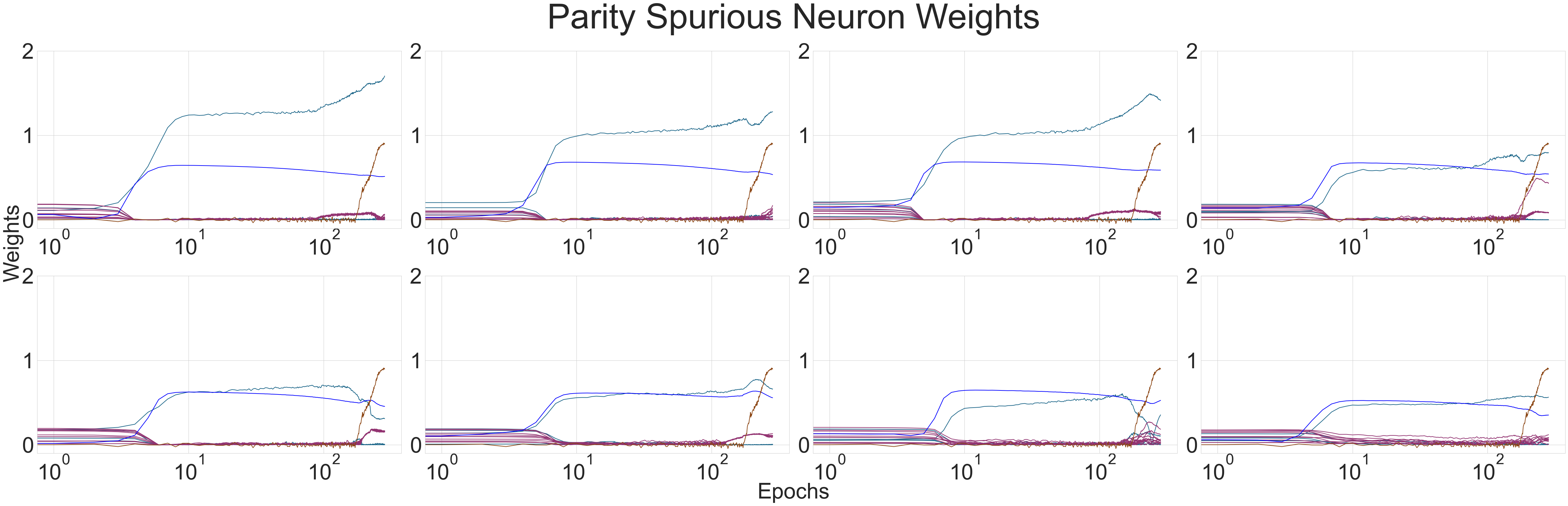}
    }
    \subfigure[Parity: All first layer neurons at $\lambda=0.95$]{
        \includegraphics[width=0.5\textwidth]{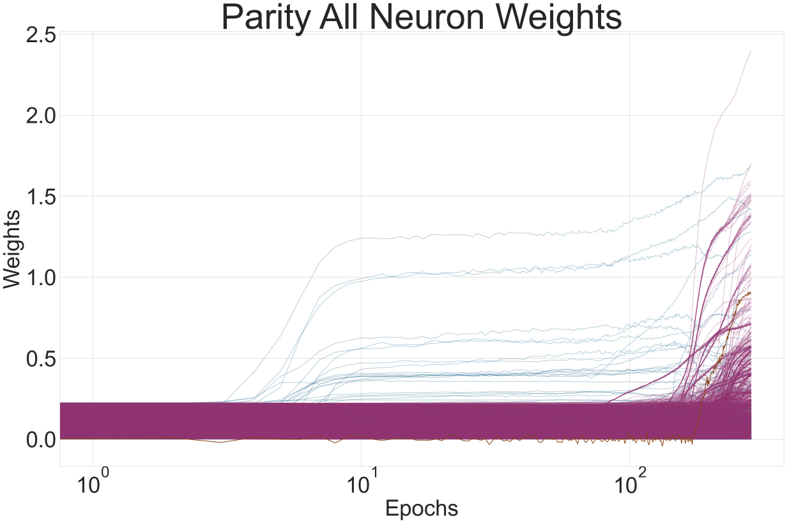}
    }
    \begin{minipage}{\textwidth}
        \centering
        \includegraphics[width=0.5\linewidth]{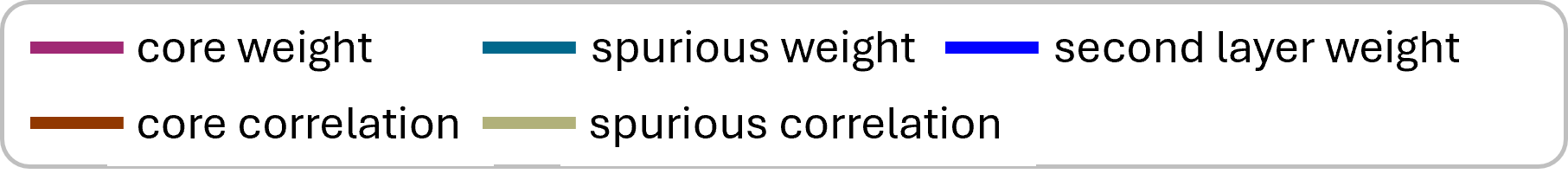}
    \end{minipage}
    \caption{\textbf{Dynamics of neurons on Parity task.}}
    \label{fig:parity_neuron}
\end{figure}

\begin{figure}
    \centering
    \subfigure[Staircase: Core Neurons at $\lambda=0.90$]{
        \includegraphics[width=1\textwidth]{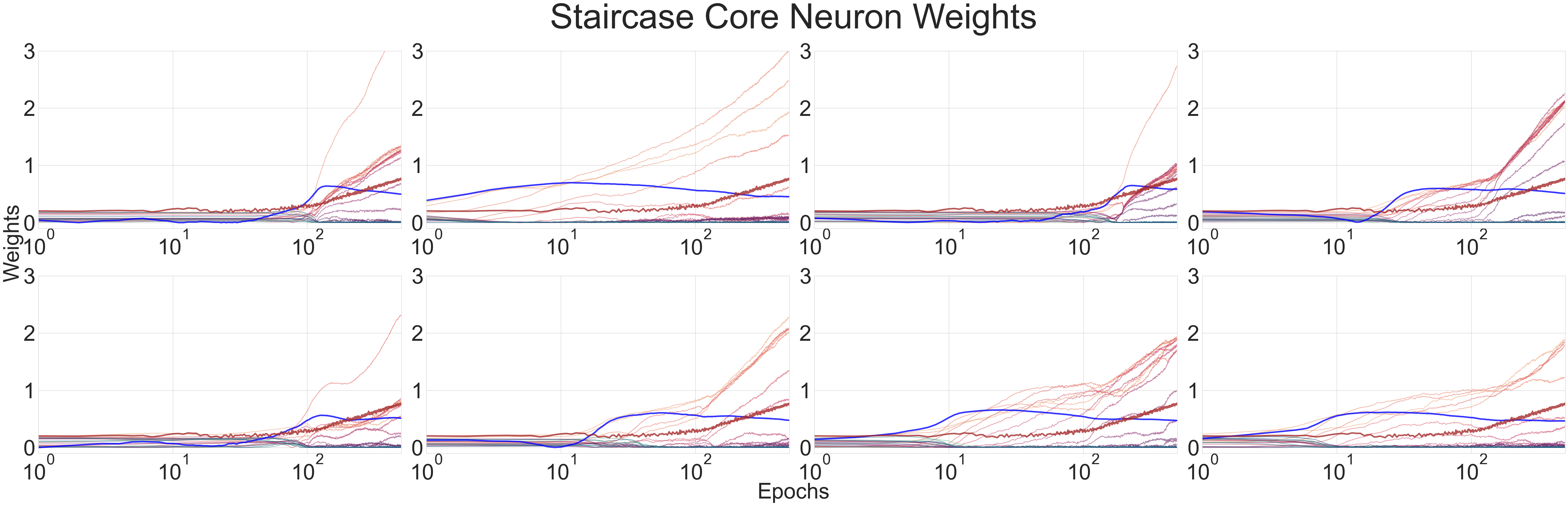}
    }
    \subfigure[Staircase: Spurious Neurons at $\lambda=0.90$]{
        \includegraphics[width=1\textwidth]{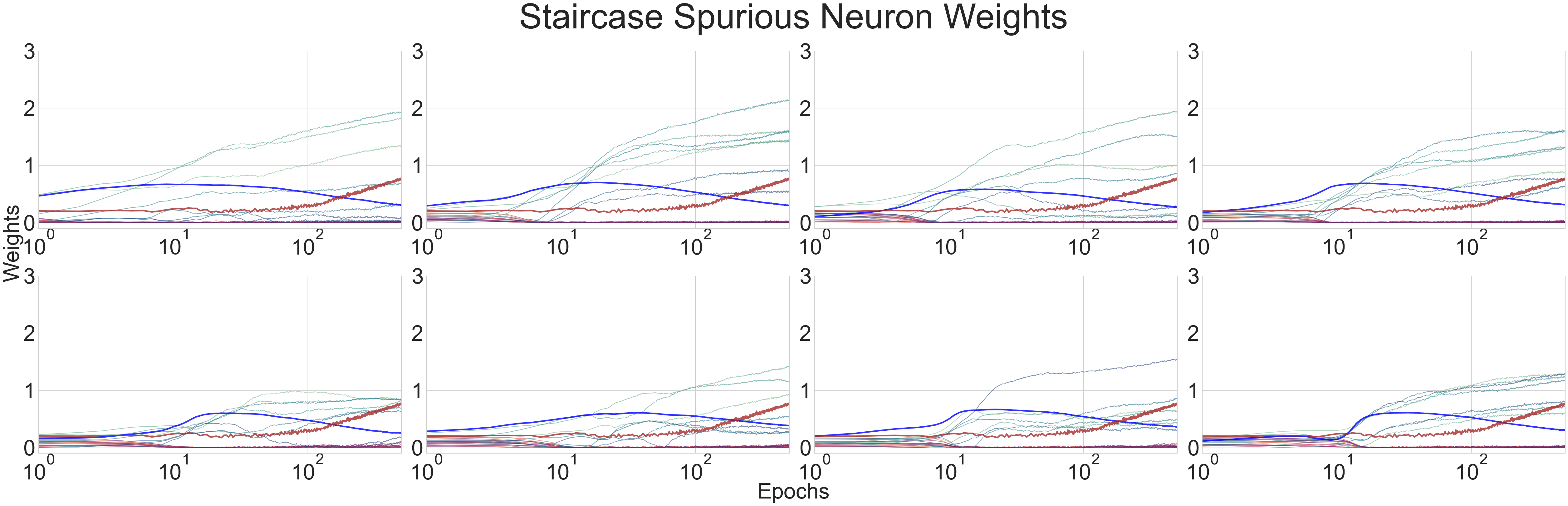}
    }
    \subfigure[Staircase: All neurons at $\lambda=0.90$]{
        \includegraphics[width=0.5\textwidth]{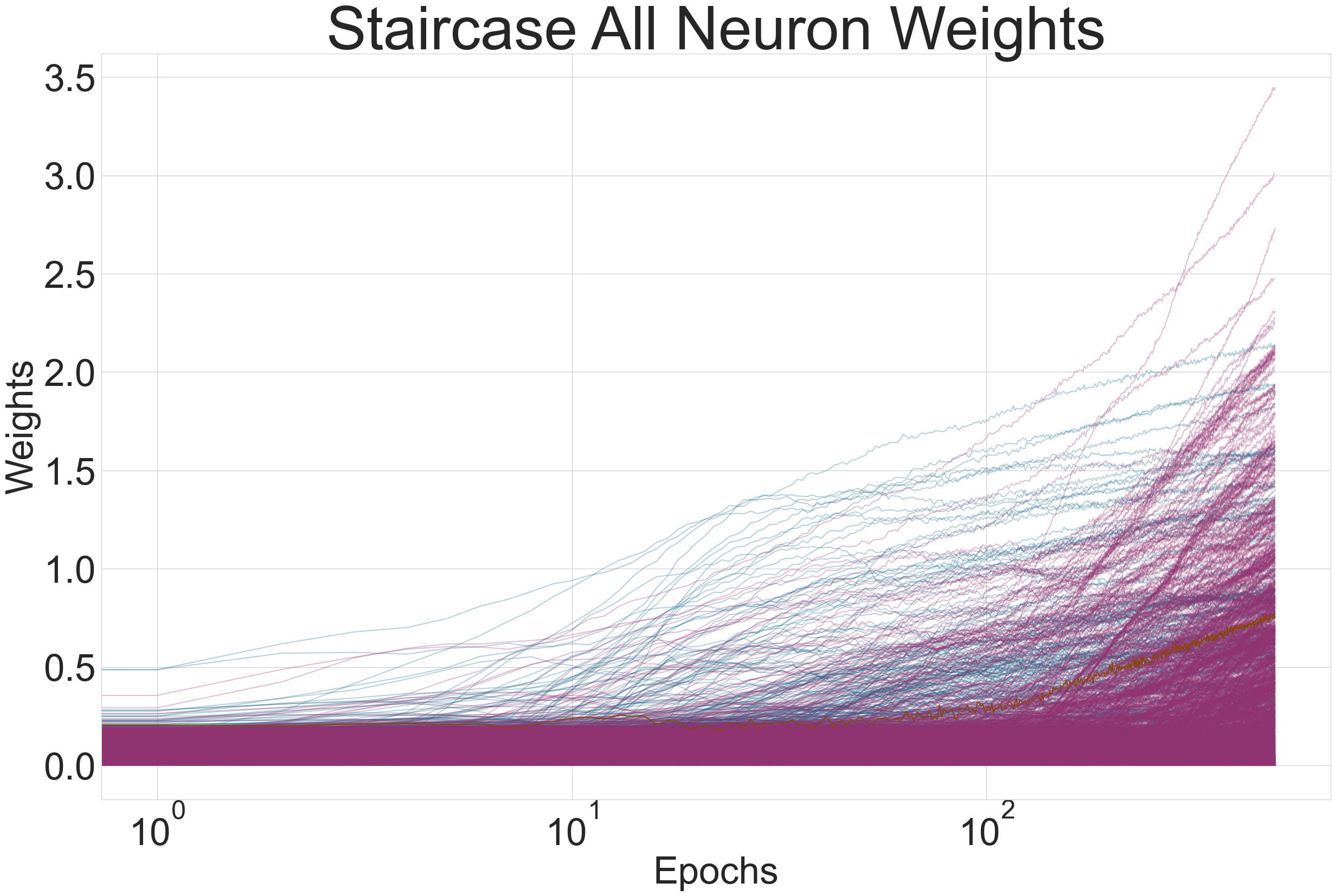}
    }
    \begin{minipage}{\textwidth}
        \centering
        \includegraphics[width=0.5\linewidth]{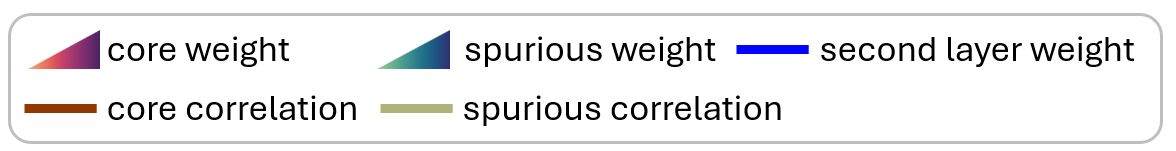}
    \end{minipage}
    \caption{\textbf{Dynamics of neurons on Staircase task.}}
    \label{fig:staircase_neuron}
\end{figure}

\subsubsection{Neurons plots on a Multi-Layer-NN}
We show a Multilayer Perceptron (MLP) may also seperate into two subnetworks. We take a 4-layer MLP trained on a staircase boolean task with $deg(f_s)=10, deg(f_c)=14, \lambda=0.9$ as an example. At each layer, starting from the bottom layer to the top layer, we recursively categorized neurons into spurious and core neurons. We observed that core neurons in the next hidden layer primarily focus on core neurons in the current hidden layer, and the same applies to spurious neurons. See \cref{fig:third_layer_neurons} for a plot on a random batch of neurons in the third layer. This discovery implies that the theory we derived in the main paper can also be extended to MLP architectures.

\begin{figure}
    \centering
    \includegraphics[width= 1\textwidth]{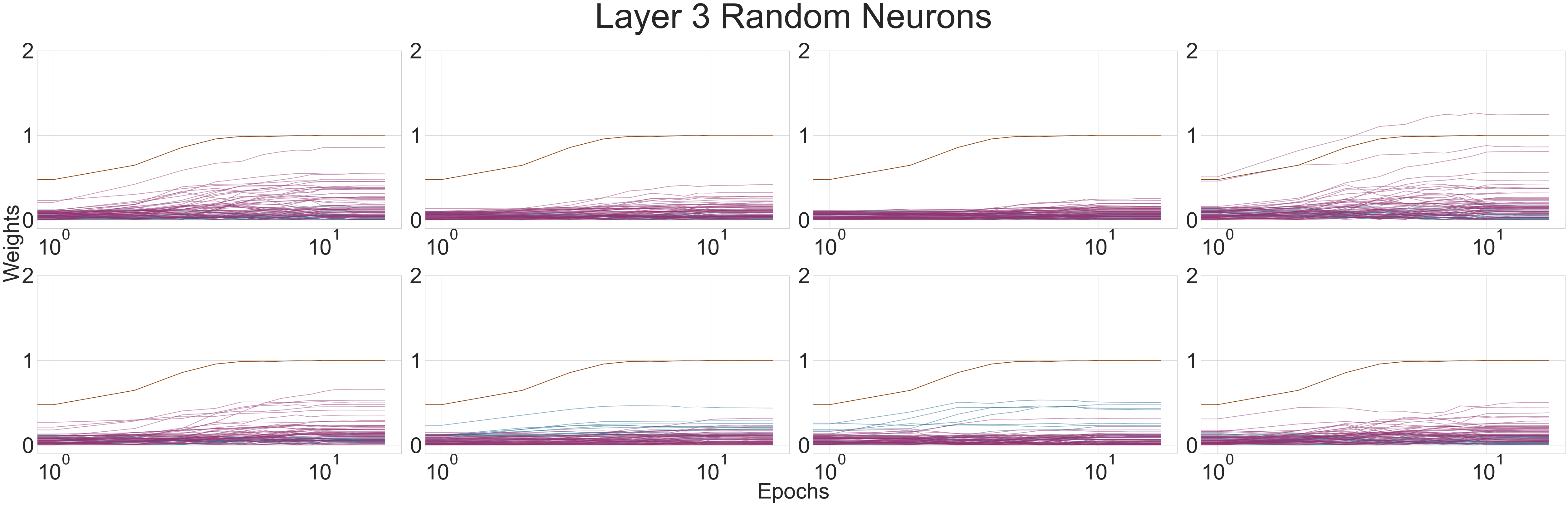}
    \caption{\textbf{Dynamics of Multi-layer neurons on Staircase task.} The plot shows the neurons in third layers are also separated into core and spurious neurons. The color of each line corresponds to the legend shown in \cref{fig:neurons} with core coordinates being the core neurons identified in the previous layer and the same for spurious neurons.}
    \label{fig:third_layer_neurons}
\end{figure}

\subsubsection{Spurious/Core neurons on Vision Datasets}
We show the finding here that neural networks trained on vision datasets are also separated/distangled into a spurious sub-network and a core sub-network. We retrain the model either on the core or  spurious feature and record the retrained correlation score. We further zero out the weights on intersected neurons which is higher than a threshold value. The result is shown in \cref{table:neuron_count}, \cref{table:neuron_performance}. 


\begin{table}[h!]
\centering
\sisetup{table-format=3.0} 
\begin{tabular}{@{}lSSSSS[table-format=1.2]@{}}
\toprule
{Dataset} & {Core Neurons} & {Spurious Neurons} & {Intersected Neurons} & {Before Retrain} \\
\midrule
Waterbirds & 135 & 72  & 6 & 1622\\
Domino     & 86  & 81  & 5  & 74 \\
CelebA     & 149 & 134 & 12 & 1449 \\
\bottomrule
\end{tabular}
\caption{Number of different types of neurons. The table shows the number of each type of neurons before and after retraining on core, spurious feature. Before Retrain: Number of activated neurons before retrain. Threshold are 0.01, 0.05, 0.01 for waterbirds, Domino and CelebA respectively.}
\label{table:neuron_count}
\end{table}

\begin{table}[h!]

\centering
\sisetup{table-format=1.3} 
\begin{tabular}{@{}lSSSS@{}}
\toprule
{Dataset} & {Core} & {Spurious} & {Core w.o Spurious} & {Spurious w.o core} \\
\midrule
Waterbirds & 0.834 & 0.826 & 0.779 & 0.816 \\
Domino     & 0.804 & 1     & 0.8   & 1     \\
CelebA     & 0.814 & 0.578 & 0.792 & 0.542 \\
\bottomrule
\end{tabular}
\caption{Performance of the Retrained model using neurons in \cref{table:neuron_count}. Core/Spurious:Retrained the model on a group balanced dataset to predict core/spurious feature. Core w.o Spurious/Spurious w.o Core: Performance of the model after removing intersected neuron weights. }
\label{table:neuron_performance}
\end{table}


\subsubsection{Last Layer Retraining reduces spurious sub-network weights}
See \cref{table:core_spurious_ratios}. We compare the last layer weights ratio between core and spurious sub-network after Lat Layer Retraining either with spurious dataset and balanced dataset. We see after retraining, the ratio increase significantly in both cases, which has also been observed in \citep{labonte_towards_2023}. 
\begin{table}[h]
\centering
\begin{tabular}{@{}lcccccc@{}}
\toprule
                     & \multicolumn{1}{l}{} & \multicolumn{1}{l}{Parity} & \multicolumn{1}{l}{} & \multicolumn{1}{l}{} & \multicolumn{1}{l}{Staircase} & \multicolumn{1}{l}{} \\ \cmidrule(l){2-7} 
\multicolumn{1}{c}{} & Core                 & Spurious                   & Ratio                & Core                 & Spurious                      & Ratio                \\ \midrule
Before Retrain       & 1.53                 & 0.97                       & 1.58                 & 1.38                 & 0.77                          & 1.79                 \\ \midrule
Retrain Spurious     & 0.57                 & 0.16                       & 3.56                 & 0.17                 & 0.08                          & 2.13                 \\ \midrule
Retrain Clean        & 0.57                 & 0.01                       & 57.00                & 0.17                 & 0.01                          & 17.00                \\ \bottomrule
\end{tabular}
\caption{The table shows the mean weights of core and spurious neurons before and after retraining with either the original spurious dataset $\dlam$ or group balanced dataset $D_{\lambda=0.5}$. Comparison of Core and Spurious Ratios in Parity and Staircase Cases before and after training with spurious dataset and clean dataset. The ratio shows $\frac{\text{Core}}{\text{Spurious}}$.}
\label{table:core_spurious_ratios}
\end{table}

\begin{figure}
    \centering
    \includegraphics[width=0.5\textwidth]{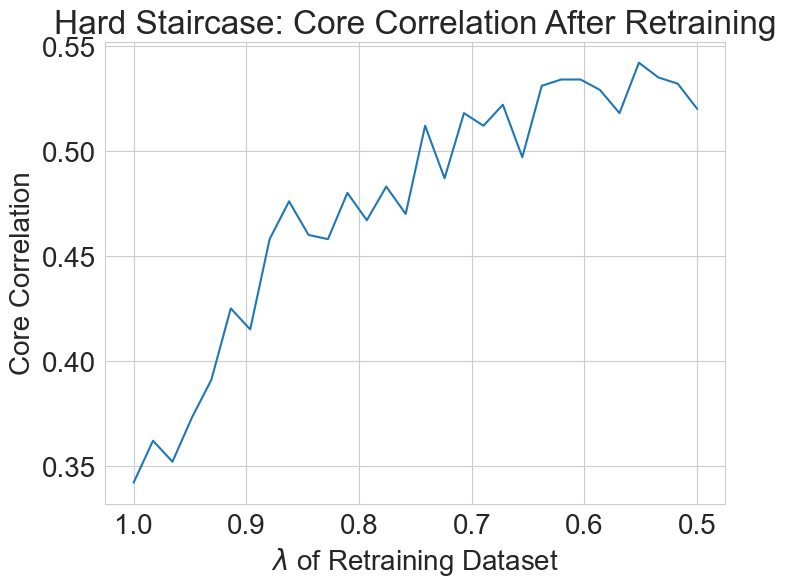}
    \caption{Hard Staircase: Core Correlation After Retrained On Dataset with Different $\lambda$. We see the result follows the same trend as observed in \citep{labonte_towards_2023} and is correctly predicted by our theory.}
    \label{fig:retrained_core_correlation_hard_staircase}
\end{figure}
\section{Discussion}
\subsection{Limitations of Existing Debiasing Algorithms}
\label{Appendix:debiasing_algorithm_limitations} 
In instances where a spurious attribute is absent, numerous debiasing algorithms \citep{liu_just_2021, liu_avoiding_2023, utama_towards_2020, nam_learning_2020, yaghoobzadeh_increasing_2021} typically follow a two-stage methodology. The first stage involves training a conventional model using Early Stop Empirical Risk Minimization (ERM). These algorithms diverge in the second stage, where each implements a distinct heuristic to distinguish and separate data from minority groups. This separation is based on the initial model, which is then utilized to either upweight or upsample these data points in the next stage. We find these methods have several inherent limitations, evident even in our toy settings: (1) identifying the right time for early stopping, (2) assessing whether the first model sufficiently identifies minority group data points, (3) determining the quantity of data points to be selected, (4) establishing the appropriate degree of upweighting for the selected points.It is also unclear whether the first model provide enough information to separate data points in the first place.
If the algorithm aims to accurately identify data points from a minority group, then we can use the Jaccard score and Containment score to evaluate their performance.

 These methods implicitly assume a distinct separation in the learning phases of spurious and core features, often influenced by the simplicity of the spurious feature and the strength of confounders. This is particularly apparent in JTT, which upweights all misclassified points in the second stage, implicitly assuming a specific temporal point where the model correlates more with spurious features than core features. This assumption holds true in cases where the spurious feature is trivial to learn compared to the core feature. such as with parity cases and popular spurious datasets like the image Domino dataset \citep{shah_pitfalls_2020}, Waterbirds \citep{sagawa_distributionally_2020}, and Color-MNIST \citep{arjovsky_invariant_2020}. However, our findings suggest that this demarcation can remain unclear throughout training, particularly with more challenging spurious features and limited datasets, as demonstrated by the limited hard staircase dataset and the hard domino dataset. As a result, these debiasing algorithms struggle to accurately distinguish minority groups from others, leading to unwanted bias in the model, as evidenced by low Jaccard scores and containment score(refer to the right two plots in Figure \cref{fig:weakness_algo}).

Unlike clustering methods such as JTT and SPARE, \citep{labonte_towards_2023} proposes a more generalized approach. Instead of segregating points into groups for upweighting based on the inferred group's size, this method selects points using the initial model, focusing on those with the highest cross-entropy or KL divergence loss from a subsequently trained model. Assuming we can accurately determine the timing for early stopping, the challenge then becomes deciding on the number of points to select. We observe that minority group samples tend to rank higher in terms of loss, as indicated by a high containment score relative to the number of selected points. However, identifying the optimal number of points without explicit knowledge of the spurious attribute can be challenging, limiting the practicality of the algorithm. Detailed statistics for four methods—JTT, Spare, Highest CE loss, Highest Disagreement Score—over six popular spurious datasets are provided in the appendix, showcasing the accuracy of these methods in identifying minority groups. For the assessment of the ranking methods, we select the number of points to be equivalent to the total number of minority points present in the spurious dataset.

Note despite not directly using a spurious attribute in training, previous algorithms often presume the availability of a validation dataset for hyperparameter tuning, which is impractical. Consequently, the reported performance typically reflects models tuned with optimal hyperparameters.

Our experiments employed the SpuCo library \citep{joshi_towards_2023}. At each epoch, we paused the initial model's training to perform group inference, following which we calculated the Jaccard and containment score to measure the accuracy of the inferred minority group against the actual minority group.


\begin{figure}
    \centering
    \includegraphics[width=1\textwidth]{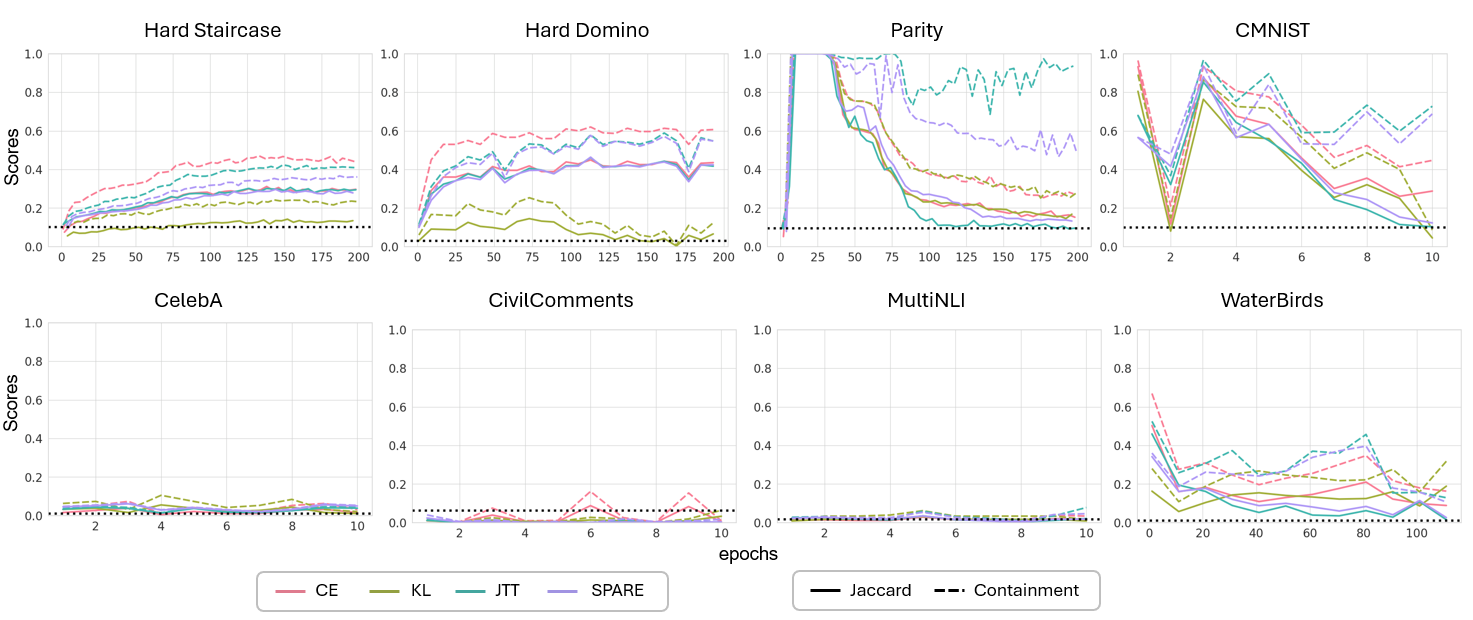}
    \caption{Performance of different debiasing methods on widely-used spurious datasets. Hard Staircase: $\lambda=0.9, \deg(f_s)=10, \deg(f_c)=14, \text{Dataset Size }60000$; Hard Domino: $\lambda=0.95$. the core feature is Cat-Dog and the spurious feature is truck-automobile from CIFAR\citep{krizhevsky_learning_nodate} dataset (see \cref{fig:hard_domino_sample_images} for an sample batch of images from this dataset). We additional report that after upsampling the inferred group (performing the second training using the best group inference result among all early stopped models), the worst group accuracy for the hard staircase and hard domino dataset are 0.42 and 0.53 respectively; Parity: $\lambda=0.9, \deg(f_s)=10, \deg(f_c)=4$; CMNIST: $\lambda=0.9$; WaterBirds: $\lambda=0.95$.}
    \label{fig:real}
\end{figure}

\begin{figure}
    \centering
    \includegraphics[width=1\textwidth]{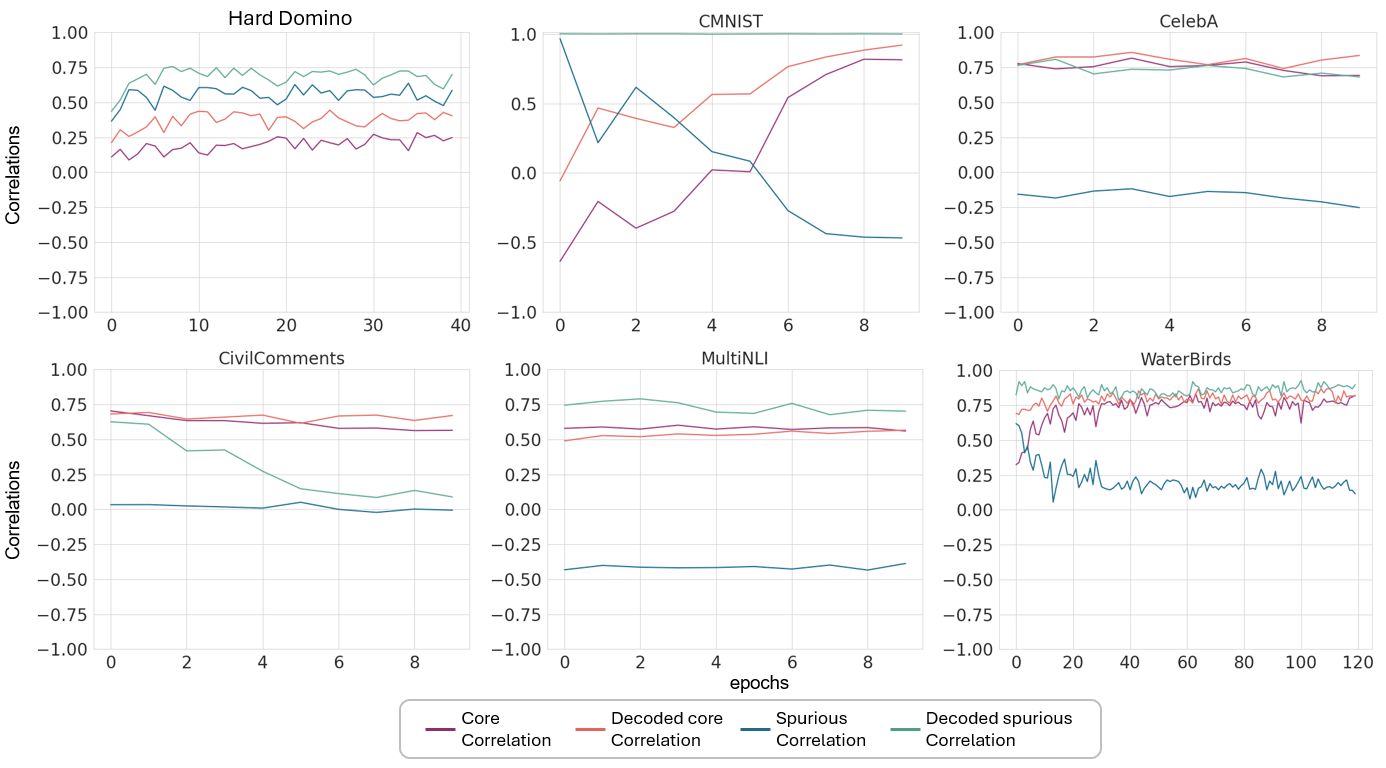}
    \caption{Training dynamics of different spurious datasets with default pretrained weights: Dataset specification is outlined in \cref{fig:real}}
    \label{fig:real-correlation-pretrained}
\end{figure}

\begin{figure}
    \centering
    \includegraphics[width=1\textwidth]{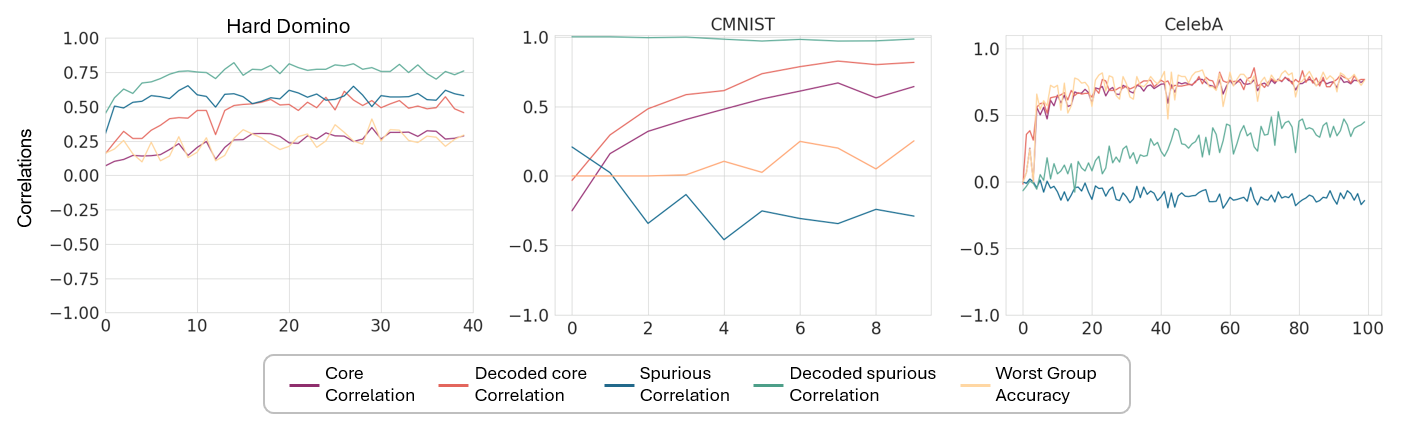}
    \caption{Training dynamics of different spurious datasets with random weights: Dataset specification is outlined in \cref{fig:real}}
    \label{fig:real-correlation-random}
\end{figure}

\begin{figure} 
    \centering
    \includegraphics[width=1\textwidth]{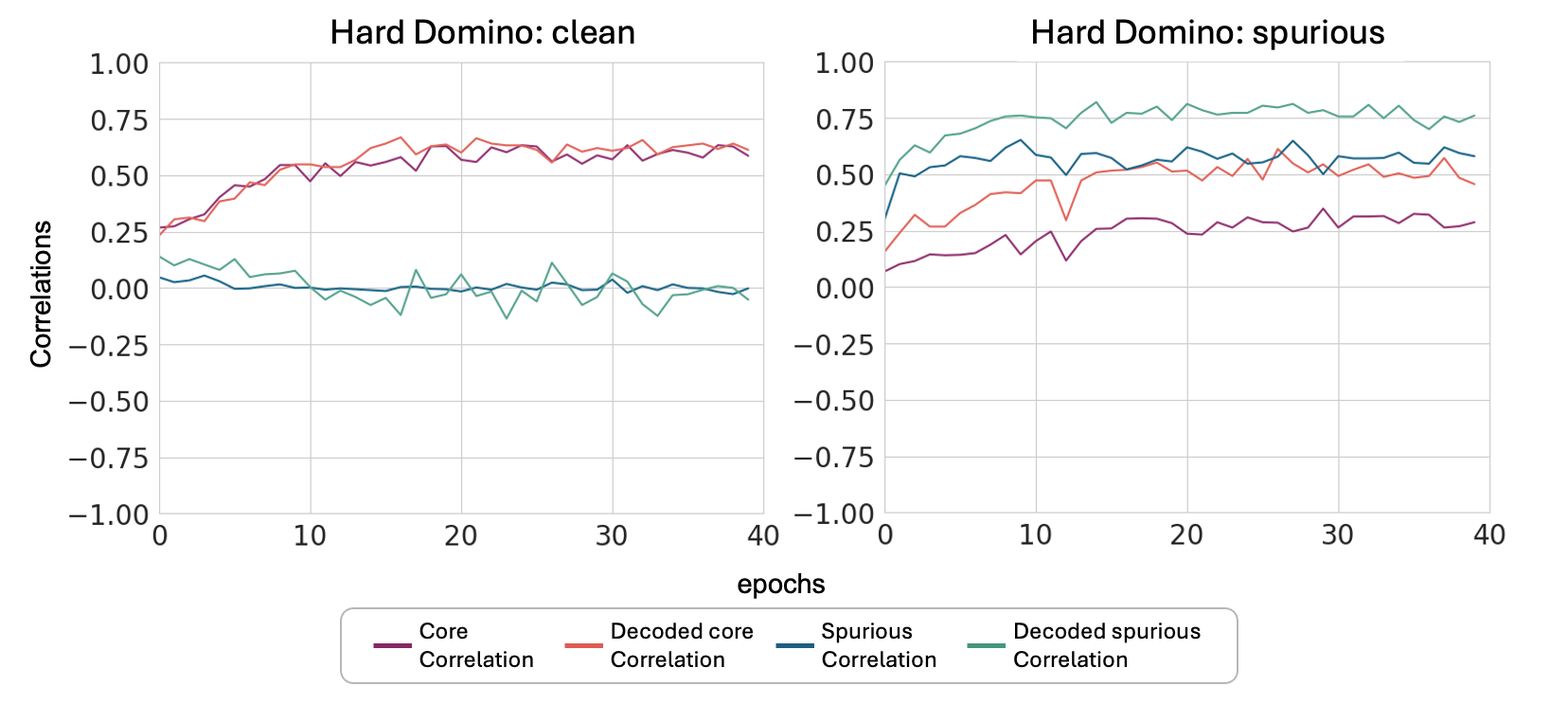}
    \caption{Core/spurious correlation and decoded correlation dynamics of the designed Hard Domino Dataset. Clean: $\lambda=0.5$; Spurious: $\lambda=0.95$. Even when the spurious feature in the domino task is challenging, it significantly influences the learning of the core feature. Thus justify its appropriateness as an benchmark that can be considered by future debiasing methods. The core correlation achieved by the model in the clean case represents the maximum possible correlation that any debiasing algorithm could achieve.}
    \label{fig:hard-domino}
\end{figure}

\begin{table}[ht]
\centering
\begin{tabular}{l|lllllll}
\hline
\textbf{Dataset/Metric}     & \textbf{clean} & \textbf{original} & \textbf{CE} & \textbf{KL} & \textbf{JTT} & \textbf{SPARE} & \textbf{clean LLR} \\ \hline
\textbf{Hard Staircase}  & 0.76           & 0.56              & -0.31       & 0.28        & 0.28         & 0.15           & 0.62               \\
\textbf{Hard Domino} & 0.70          & 0.375             & -0.22       & 0            & 0.15         & 0.02           & 0.55               \\ \hline
\end{tabular}
\caption{End Performance of Different Debiasing methods on designed experiments: \textbf{clean}: Trained from scratch with $\lambda=0.5$ which can be seen as a optimal case. \textbf{original}: Before Retraining with $\lambda=0.95$ for Hard Domino task and $\lambda=0.9$ for staircase task.\textbf{clean LLR}: Last Layer Retraining with $\lambda=0.5$ dataset. The value here denote core correlation after retraining. It is observed that previous debiasing algorithms either failed on the designed spurious task or cause failure to core feature learning.}
\label{tab:end_performance}
\end{table}

\begin{figure}
    \centering
    \includegraphics[width=1\textwidth]{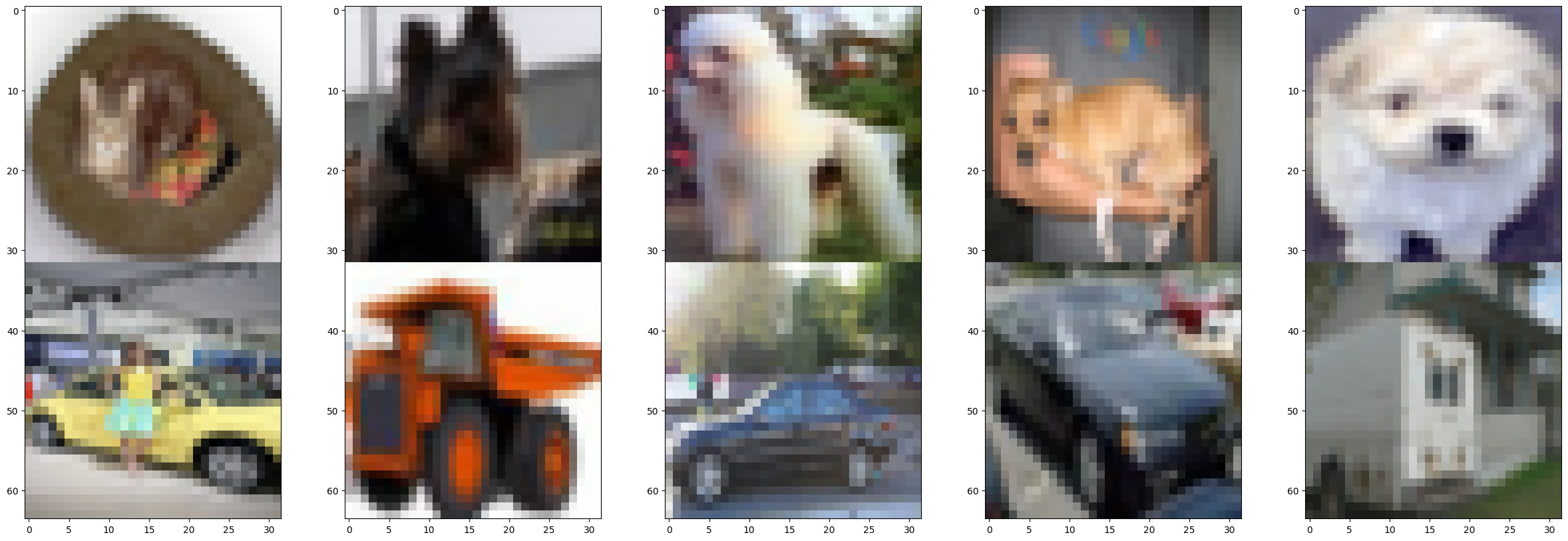}
    \caption{A sample batch of images from the constructed Hard Domino dataset.}
    \label{fig:hard_domino_sample_images}
\end{figure}

\subsection{Discussion on spurious real dataset}
\label{Appendix:discussion_real_dataset}
Real-world datasets typically used to study spurious correlations, such as MultiNLI\citep{williams_broad-coverage_2018} and Civilcomment\citep{duchene_benchmark_2023}, often do not meet the properties of the boolean spurious datasets where core feature and spurious feature are disentangled and realizable. In these datasets, spurious attributes are intricately intertwined at the word level, such that the removal of a negative word can significantly alter a sentence's semantic meaning. Furthermore, it has been reported that the labels in these datasets are not entirely clear-cut and may be inherently ambiguous, posing challenges for semantic labeling even for human annotators, thereby violating the realizability condition. We observe that the training dynamics in such real-world datasets are diversified and not fully understood, as indicated by \citep{izmailov_feature_2022, joshi_towards_2023} (refer to \cref{fig:real-correlation-pretrained}, \cref{fig:real-correlation-random}). It lacks justifications whether these datasets are suitable to be used in studying spurious correlation. 

As previous studies \citep{labonte_towards_2023, idrissi_simple_2022} have repeatedly shown, class-balanced training achieves comparable performance to other, more sophisticated debiasing algorithms, or even group-balanced training such as DRO \citep{sagawa_distributionally_2020} on real dataset. The analysis of Jaccard and containment scores reveals that the debiasing methods tested exhibit poor performance on all the tested real spurious dataset, casting doubt on their effectiveness for enhancing the second model. Therefore, the potential of these methods to improve core feature learning is questionable. The pervasive use of pretrained models further complicates the evaluation of debiasing algorithms. A deeper understanding of the complex nature of real-world data, an area that remains largely unexplored, is crucial for a comprehensive understanding of the training dynamics in these scenarios.

\end{document}